\documentclass{article}

\usepackage{neurips_2022}

\usepackage[utf8]{inputenc} 
\usepackage[T1]{fontenc}    
\usepackage{hyperref}       
\usepackage{url}            
\usepackage{booktabs}       
\usepackage{amsfonts}       
\usepackage{nicefrac}       
\usepackage{microtype}      
\usepackage{xcolor}         

\usepackage[utf8]{inputenc} %
\usepackage[T1]{fontenc}    %
\usepackage{hyperref}       %
\usepackage{url}            %
\usepackage{booktabs}       %
\usepackage{amsfonts}       %
\usepackage{nicefrac}       %
\usepackage{microtype, color}      %

\usepackage{tabulary}
\usepackage{tabularx}
\usepackage{array}
\usepackage{algorithm, algorithmicx}
\usepackage[noend]{algpseudocode}
\usepackage{amsmath}
\usepackage{amssymb}
\usepackage{amsthm}
\usepackage{bbm}
\usepackage{graphicx}
\usepackage{indentfirst}

\newlength\savewidth\newcommand\shline{\noalign{\global\savewidth\arrayrulewidth
		\global\arrayrulewidth 1pt}\hline\noalign{\global\arrayrulewidth\savewidth}}
\newcommand{\tablestyle}[2]{\setlength{\tabcolsep}{#1}\renewcommand{\arraystretch}{#2}\centering\footnotesize}

\newtheorem{theorem}{Theorem}[section]
\newtheorem{claim}[theorem]{Claim}
\newtheorem{definition}[theorem]{Definition}
\newtheorem{lemma}[theorem]{Lemma}

\newtheorem{proposition}[theorem]{Proposition}
\newtheorem{remark}[theorem]{Remark}

\newtheorem{assumption}[theorem]{Assumption}

\makeatletter
\newtheorem*{rep@definition}{\rep@title}
\newcommand{\newrepdefinition}[2]{%
	\newenvironment{rep#1}[1]{%
		\def\rep@title{#2 \ref{##1}}%
		\begin{rep@definition}}%
		{\end{rep@definition}}}
\makeatother

\newrepdefinition{definition}{Definition}
\newrepdefinition{lemma}{Lemma}
\newrepdefinition{proposition}{Proposition}

\newcommand{\loss}[1]{\mathcal{L}_{#1}}
\newcommand{\reg}{\sigma} 
\newcommand{\Exp}[2]{\mathop{\mathbb{E}}_{#1}\left[#2\right]}
\newcommand{\E}[1]{\mathop{\mathbb{E}}_{#1}}
\newcommand{\Esymbol}{\mathrm{\mathbb{E}}}
\newcommand{\ppos}{P_+}
\newcommand{\pdata}{P_\data}
\newcommand{\norm}[1]{\left\lVert#1\right\rVert}
\newcommand{\imatrix}{I}
\newcommand{\Real}{\mathbb{R}}
\newcommand{\data}{\mathcal{X}}
\newcommand{\integer}{\mathbb{Z}^+}
\newcommand{\lossscl}{\mathcal{L}_{\textup{scl}}}
\DeclareMathOperator{\Tr}{Tr}
\newcommand{\tildeF}{\widetilde{F}}
\newcommand{\norA}{\bar{{A}}}
\newcommand{\Ndata}{N}
\newcommand{\id}[1]{\mathbbm{1}\left[#1\right]}
\newcommand{\one}{\mathbf{1}}
\newcommand{\cluster}{\tau}
\newcommand{\clusterf}[1]{\cluster({#1})}
\newcommand{\y}{y}
\newcommand{\yf}[1]{\y({#1})}
\newcommand{\eigval}{\lambda}
\newcommand{\headmatrix}{B}
\newcommand{\head}{b}
\newcommand{\nclass}{r}
\newcommand{\ncluster}{m}
\newcommand{\idvec}[1]{e_{#1}}
\newcommand{\laplacian}{\mathcal{L}}
\newcommand{\source}{S}
\newcommand{\target}{T}
\newcommand{\precond}{\Sigma}
\newcommand{\psource}{\mathcal{P}_\source}
\newcommand{\ptarget}{\mathcal{P}_\target}
\newcommand{\pred}{g}
\DeclareMathOperator*{\argmax}{arg\,max}

\newcommand{\Prob}{\mathcal{P}_\data}

\newcommand{\graph}{G}

\newcommand{\Err}{\mathcal{E}}
\newcommand{\wvec}{g}
\newcommand{\clusterset}{C}
\newcommand{\expansion}{\gamma}

\def\shownotes{0}  \ifnum\shownotes=1
\newcommand{\authnote}[2]{{[#1: #2]}}
\else
\newcommand{\authnote}[2]{}
\fi

\def\shownotes{0}  \ifnum\shownotes=1
\newcommand{\authornotenonurgent}[2]{{[#1: #2]}}
\else
\newcommand{\authornotenonurgent}[2]{}
\fi

\title{Beyond Separability: Analyzing the Linear Transferability of Contrastive Representations to Related Subpopulations}

\author{%
	David S.~Hippocampus\thanks{Use footnote for providing further information
		about author (webpage, alternative address)---\emph{not} for acknowledging
		funding agencies.} \\
	Department of Computer Science\\
	Cranberry-Lemon University\\
	Pittsburgh, PA 15213 \\
	\texttt{hippo@cs.cranberry-lemon.edu} \\
}

\begin{document}

	\maketitle

	\begin{abstract}

Contrastive learning is a highly effective method for learning representations from unlabeled data. Recent works show that contrastive representations can transfer across domains, leading to simple state-of-the-art algorithms for unsupervised domain adaptation. In particular, a linear classifier trained to separate the representations on the source domain can also predict classes on the target domain accurately, even though the representations of the two domains are far from each other.
We refer to this phenomenon as \textit{linear transferability}. This paper analyzes when and why contrastive representations exhibit linear transferability in a general unsupervised domain adaptation setting. 
We prove that linear transferability can occur when data from the same class in different domains (e.g., photo dogs and cartoon dogs) are more related with each other than data from different classes in different domains (e.g., photo dogs and cartoon cats) are. 
Our analyses are in a realistic regime where the source and target domains can have unbounded density ratios and be weakly related, and they have distant representations across domains. 
	\end{abstract}


\section{Introduction}
In recent years, contrastive learning and related ideas have been shown to be highly effective for representation learning~\citep{chen2020simple, chen2020big, he2020momentum, caron2020unsupervised, chen2020improved,gao2021simcse, su2021tacl,chen2020exploring}.
Contrastive learning trains representations on \emph{unlabeled data} by encouraging positive pairs (e.g., augmentations of the same image) to have closer representations than negative pairs (e.g., augmentations of two random images).  The learned representations are almost \textit{linearly separable}: one can train a linear classifier on top of the fixed representations and achieve strong performance on many natural downstream tasks~\citep{chen2020simple}.
Prior theoretical works analyze contrastive learning by proving that semantically similar
datapoints (e.g., datapoints from the same class) are mapped to geometrically nearby representations~\citep{arora2019theoretical,tosh2020contrastive,tosh2021contrastive,haochen2021provable}. In other words, representations form clusters in the Euclidean space that respect the semantic similarity; therefore, they are linearly separable for downstream tasks where datapoints in the same semantic cluster have the same label. 


\begin{figure}
	\includegraphics[width=1.0\textwidth]{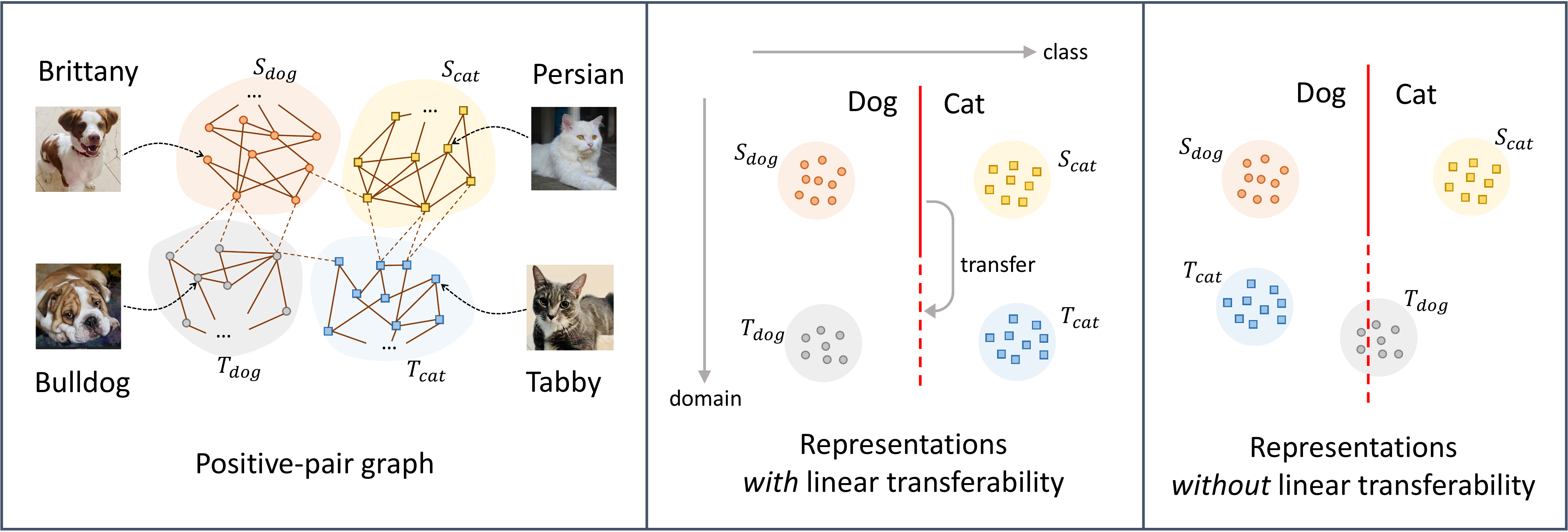}
	\caption{ \textbf{The linear transferability of representations.}
		We demonstrate the linear transferability of representations when the unlabeled data contains images of two breeds of dogs (Brittanys, Bulldogs) and two breeds of cats (Persians, Tabbies).
		\textbf{Left:} A visualization of the positive-pair graph with four semantic clusters. Inter-cluster edges (dashed) have a much smaller weight than intra-cluster edges (solid). Inter-cluster edges between two breeds of dogs (or cats) have more weight than that between a dog cluster and a cat cluster.
		\textbf{Middle and right:} A visualization of two different types of representations: both have linear separability, but only the middle one has linear transferability. 
		The red line is the decision boundary of a dog-vs-cat linear classifier trained in the representation space on \emph{labeled} Brittanys ($S_\text{dog}$) vs. Persians ($S_\text{cat}$) images. The representation has linear transferability if this classifier is  accurate on \emph{unlabeled} Bulldogs ($T_\text{dog}$) vs. Tabbies ($T_\text{cat}$) images.	\label{figure:main_figure}
	}
\end{figure}

Intriguingly, recent empirical works show that contrastive representations carry richer information \textit{beyond} the cluster memberships---they can transfer across domains in a linear way as elaborated below. Contrastive learning is used in many unsupervised domain adaptation algorithms\citep{ thota2021contrastive, sagawa2021wilds} and the transferability leads to simple state-of-the-art algorithms~\citep{shen2021does, park2020jointcontrastive, wang2021crossdomain}. 
In particular, \citet{shen2021does} observe that the relationship between two clusters can be captured by their relative positions in the representation space. 
 For instance, as shown in Figure~\ref{figure:main_figure} (middle), suppose $S_\text{dog}$ and $S_\text{cat}$ are two classes in a \textit{source} domain (e.g., Brittany dogs and Persian cats), and $T_\text{dog}$ and $T_\text{cat}$ are two classes in a \textit{target} domain (e.g., Bulldogs and Tabby cats). A \textit{linear} classifier trained to separate the representations of $S_\text{dog}$ and $S_\text{cat}$ turns out to classify $T_\text{dog}$ and $T_\text{cat}$ as well. This suggests the four clusters of representations are not located in the Euclidean space randomly (e.g., as in Figure~\ref{figure:main_figure} (right)), but rather in a more aligned position as in Figure~\ref{figure:main_figure} (middle). 
 We refer to this phenomenon as the \textit{linear transferability} of contrastive representations. 




This paper analyzes when and why contrastive representations exhibit linear transferability in a general unsupervised domain adaptation setting. 
Evidently, linear transferability can only occur when clusters corresponding to the same class in two domains (e.g., Brittany dogs and Bulldogs) are somewhat \textit{related} with each other. 
Somewhat surprisingly, we found that a weak relationship suffices: linear transferability occurs as long as corresponding classes in different domains are more related than different classes in different domains. 
Concretely, under this assumption (Assumptions~\ref{assumption:uda} or~\ref{assumption:uda_multistep_interclass}), a linear head learned with labeled data on one domain (Algorithm~\ref{algorithm:1}) can successfully predict the classes on the other domain (Theorems~\ref{theorem:uda} and~\ref{theorem:uda_multistep}). Notably, our analysis provably shows that representations from contrastive learning do not only encode cluster identities but also capture the inter-cluster relationship, hence explains the empirical success of contrastive learning for domain adaptation.

 
Compared to previous theoretical works on unsupervised domain adaptation~\citep{shimodaira2000improving,huang2006correcting,sugiyama2007covariate, gretton2008covariate, ben2010theory, mansour2009domain, kumar2020gradual,chen2020selftraining, cai2021theory}, our results analyze a modern, practical algorithm with weaker and more realistic assumptions.
We do not require bounded density ratios or overlap between the source and target domains, which were assumed in some classical works~\citep{sugiyama2007covariate, ben2010theory, zhang2019bridging, zhao2019learning}. 
Another line of prior works~\citep{kumar2020gradual,chen2020selftraining} assume that data is Gaussian or near-Gaussian, whereas our result allows more general data distribution.  \cite{cai2021theory} analyze pseudolabeling algorithms for unsupervised domain adaptation, but require that the same-class cross-domain data are more related with each other 
(i.e., more likely to form positive pairs) 
than cross-class same-domain data are. 
We analyze a contrastive learning algorithm with strong empirical performance, and only require that the same-class cross-domain data are more related with each other than cross-class \textit{cross-domain} data, which is intuitively and empirically more realistic as shown in~\citet{shen2021does}. (See related work and discussion below Assumption~\ref{assumption:uda} for details). 

Technically, we significantly extend the framework of~\citet{haochen2021provable} to allow distribution shift---our setting only has labels on one subpopulation of the data (the source domain). Studying transferability to unlabeled subpopulations requires both novel assumptions (Assumptions~\ref{assumption:uda} and~\ref{assumption:uda_multistep_interclass}) and novel analysis techniques (as discussed in Section~\ref{section:sketch}). 

Our analysis also introduces a variant of the linear probe—instead of training the linear head with the logistic loss, we learn it by directly computing the average representations within a class, multiplied by a preconditioner matrix (Algorithm~\ref{algorithm:1}). We empirically test this linear classifier on benchmark datasets and show that it achieves superior domain adaptation performance in Section~\ref{sec:experiments}.




\paragraph{Additional Related Works.}
A number of papers have analyzed the linear separability of representations from contrastive learning~\citep{arora2019theoretical,tosh2020contrastive,tosh2021contrastive,haochen2021provable} and self-supervised learning~\citep{lee2020predicting}, whereas we analyze the linear transferability. 
~\citet{shen2021does} also analyze the linear transferability but only for toy examples where the data is generated by a stochastic block model. Their technique requires a strong symmetry of the positive-pair graph (which likely does not hold in practice) so that top eigenvectors can be analytically derived.  Our analysis is much more general and does not rely on explicit, clean form of the eigenvectors (which is impossible for general graphs).

Empirically, pre-training on a larger unlabeled dataset and then fine-tuning on a smaller labeled dataset is one of the most successful approaches for handling distribution shift~\citep{blitzer2007biographies, ziser2018deep, ziser2017neural, ben2020perl, chen2012marginalized, xie2020n, jean2016combining, hendrycks2020pretrained, kim2022broad, kumar2022fine, sagawa2021wilds, thota2021contrastive, shen2021does}. Recent advances in the scale of unlabeled data, such as in BERT and CLIP, have increased the importance of this approach~\citep{wortsman2022model, wortsman2021robust}. Despite the empirical progress, there has been limited theoretical understanding of why pre-training helps domain shift. Our work provides the first analysis that shows pre-trained representations with a supervised linear head trained on one domain can provably generalize to another domain.

\section{Preliminaries}\label{section:preliminaries}

\newcommand{\ww}[1]{w(#1)}
\newcommand{\phiavg}{\phi}
\newcommand{\phimax}{\bar{\phi}}
\newcommand{\phimin}{\underline{\phi}}

In this section, we introduce the contrastive loss, define the positive-pair graph, and introduce the basic assumptions on the clustering structure in the positive-pair graph. 

\paragraph{Positive pairs.} 
Contrastive learning algorithms rely on the notion of ``positive pairs'', which are pairs of semantically similar/related data. 
Let $\data$ be the set of population data and $\ppos$ be the distribution of positive pairs of data satisfying $\ppos(x,x') = \ppos(x', x)$ for any $x, x'\in\data$. We note that though a positive pair typically consists of semantically related data, the vast majority of semantically related pairs are \textit{not} positive pairs. In the context of computer vision problems~\citep{chen2020simple}, these pairs are usually generated via data augmentation on the same image.

For the ease of exposition, we assume $\data$ is a finite but large set (e.g., all real vectors in $\Real^d$ with bounded precision) of size $N$. We use $\pdata$ to denote the marginal distribution of $\ppos$, i.e.,  $\pdata(x) := \sum_{x'\in\data} \ppos(x, x')$. 
Following the terminology in the literature~\citep{arora2019theoretical}, we call $(x, x')$ a ``negative pair'' if $x$ and $x'$ are independent random samples from $\pdata$. 


\paragraph{Generalized spectral contrastive loss.} 
Contrastive learning trains a representation function (feature extractor) by minimizing a certain form of contrastive loss.
Formally, let $f:\data\rightarrow\Real^k$ be a mapping from data to $k$-dimensional features. 
In this paper, we consider a more general version of the spectral contrastive loss proposed in~\citet{haochen2021provable}. 
Let $\imatrix_{k\times k}$ be the $k$-dimensional identity matrix. We consider the following loss with regularization strength $\reg > 0$:
	\begin{align}
		\loss{\reg}(f) = \underset{(x, x^+)\sim \ppos}{\Esymbol}\big[\norm{f(x)-f(x^+)}_2^2\big] + \reg \cdot R(f), \label{eqn:12}
	\end{align}
where the regularizer is defined as
\begin{align}	
R(f) = 	\Big\|\underset{x\sim \pdata}{\Esymbol} \big[f(x)f(x)^\top\big]- \imatrix_{k\times k}\Big\|_F^2.
\end{align}
The loss $\loss{\reg}$ intuitively minimizes the closeness of positive pairs via its first term,  while regularizing the representations' covariance to be identity,  avoiding all the representations to collapse to the same point. Simple algebra shows that $\loss{\reg}$ recovers the original spectral contrastive loss when $\reg=1$ (see Proposition~\ref{proposition:generalized_loss} for a formal derivation).
We note that this loss is similar in spirit to the recently proposed Barlow Twins loss~\citep{zbontar2021barlow}. 

\paragraph{The positive-pair graph. }
One useful way to think of positive pairs is through a graph defined by their distribution. Let the \emph{positive-pair graph} be a weighted undirected graph $\graph(\data, w)$ such that the vertex set is $\data$, and for $x, x'\in\data$, the undirected edge $(x, x')$ has weight $\ww{x,x'}=\ppos(x, x')$. 
This graph was introduced by~\citet{haochen2021provable} as the augmentation graph when the positive pairs are generated from data augmentation. We introduce a new name to indicate the more general applications of the graph into other use cases of contrastive learning (e.g. see~\citet{gao2021simcse}).
We use $\ww{x}=\pdata(x) = \sum_{x'\in\data} \ww{x,x'}$ to denote the total weight of edges connected to a vertex $x$. 
We call $\norA\in\Real^{\Ndata\times\Ndata}$ the \textit{normalized adjacency matrix} of $\graph(\data, w)$ if $\norA_{xx'}={\ww{x,x'}}/{\sqrt{\ww{x}\ww{x'}}}$,\footnote{We index $\norA$ by $(x, x')\in\data\times\data$. Generally, we will index the $\Ndata$-dimensional axis of an array by $x\in\data$.} and call $\laplacian := \imatrix_{\Ndata\times\Ndata} -\norA$ the \textit{Laplacian} of $\graph(\data, w)$.

\subsection{Clustering assumptions}
Previous work accredits the success of contrastive learning to the clustering structure of the positive-pair graph---because the positive pairs connect data with similar semantic contents, the graph can be partitioned into many semantically meaningful clusters. 
To formally describe the clustering structure of the graph, we will use the notion of expansion. 
For any subset $A$ of vertices, let $\ww{A}\triangleq \sum_{x\in A} \ww{x}$ be the total weights of vertices in $A$. For any subsets $A, B$ of vertices, let $\ww{A, B} \triangleq \sum_{x\in A, x'\in B} \ww{x,x'}$ be the total weights between set $A$ and $B$. We abuse notation and use $\ww{x, B}$ to refer to $\ww{\{x\}, B}$ when the first set is a singleton.
\begin{definition}[Expansion] Let $A, B$ be two disjoint subsets of $\data$. We use $\phiavg(A, B)$, $\phimax(A,B)$ and $\phimin(A,B)$ to denote the expansion, max-expansion and min-expansion from $A$ to $B$ respectively, defined as
	\begin{align}
		\phiavg(A, B) = \frac{w(A, B)}{w(A)}\,,\quad\quad		
		\phimax(A, B) = \max_{x\in A}\frac{w(x, B)}{w(x)}\,,\quad\quad \phimin(A, B) = \min_{x\in A}\frac{w(x, B)}{w(x)}
	\end{align}
	Note that $\phimin(A, B) \le \phiavg(A, B) \le \phimax (A, B)$. 
\end{definition}

Intuitively, $\phi(A, B)$ is the average proportion of edges adjacent to vertices in $A$ that go to $B$, whereas the max-(min-)expansion is an upper (lower) bound of this proportion for each $x\in A$.

Our basic assumption on the positive-pair graph is that the vertex set $\data$ can be partitioned into $m$ groups $C_1,\dots, C_m$ with small connections (expansions) across each other. 
\begin{assumption}[Cross-cluster connections]\label{assumption:separation}
	For some $\alpha \in (0,1)$, we assume that the vertices of the positive-pair graph $G$ can be partition into $m$ disjoint clusters $C_1,\dots, C_m$ such that for any $i\in [m]$,
	\begin{align}
		\phimax(C_i, \data\backslash C_i) \le \alpha
	\end{align}
\end{assumption} 

We will mostly work with the regime where $\alpha\ll 1$. Intuitively, each $C_i$ corresponds to all the data with a certain semantic meaning. For instance, $C_i$ may contain dogs from a certain breed. Our assumption is slightly stronger than in~\citet{haochen2021provable}. In particular, they assume that the average expansions cross clusters is small, i.e.,
$
\sum_{i\in[\ncluster]} \phi(C_i, \data\backslash C_i) \cdot \ww{C_i} \le \alpha,
$
whereas we assume that the max-expansion is smaller than $\alpha$ for each cluster.
In fact, since $\sum_{i\in[\ncluster]}\ww{C_i}=1$ and $\phi(C_i, \data\backslash C_i)\le \phimax(C_i, \data\backslash C_i)$, Assumption~\ref{assumption:separation} directly implies their assumption. However, we note that Assumption~\ref{assumption:separation} is still realistic in many domains. For instance, any bulldog $x$ has way more neighbors that are still bulldogs than neighbors that are Brittany dog, which suggests the max-expansion between bulldogs and Brittany dogs is small. 

We  also introduce the following assumption about intra-cluster expansion that guarantees each cluster can not broken into two well-separated sub-clusters.  
\begin{assumption}[Intra-cluster conductance]\label{assumption:intra_class_conductance}
	For all $i\in [m]$,  assume the conductance of the subgraph restricted to $C_i$ is large, that is, every subset $A$ of $C_i$ with at most half the size of $C_i$ expands to the rest:
\begin{align}
		\forall A \subset C_i \textup{ satisfying } w(A)\le w(C_i)/2, ~ \phi(A, C_i\backslash A) \ge \expansion .
\end{align}
\end{assumption}
We have $\gamma < 1$ and we typically work with the regime where $\gamma$ is decently large (e.g., $\Omega(1)$, or inverse polynomial in dimension)\footnote{E.g., suppose each cluster's distribution is a Gaussian distribution with covariance $I$, and the data augmentation is Gaussian blurring with a covariance $\frac{1}{d}\cdot I$, then the intra-cluster expansion is $\Omega(1)$ by Gaussian isoperimetric inequality~\citep{bobkov1997isoperimetric}. The same also holds with a Lipschitz transformation of Gaussian.} and  much larger than the cross-cluster connections $\alpha$. This is the same regime where prior work~\citet{haochen2021provable} guarantees the representations of clusters are linearly separable.


We also remark that all the assumptions are on the population positive-pair graph, which is sparse but has reasonable connected components (as partially evaluated in~\cite{wei2020theoretical}). The rest of the paper assumes access to population data, but the main results can be extended to polynomial sample results by levering a model class for representation functions with bounded Rademacher complexity as shown in ~\citet{haochen2021provable}.\footnote{In contrast, the positive-graph built only on empirical examples will barely have any edges, and does not exhibit any nice properties. However, the sample complexity bound does not utilize the empirical graph at all. } 
\section{Main Results on Linear Transferability}
\label{section:main_results} 
In this section, we analyze the \emph{linear transferability} of contrastive representations by showing that representations encode information about the relative strength of relationships between clusters. 

Let $\source$ and $\target$ be two disjoint subsets of $\data$, each formed by $r$ clusters corresponding to $r$ classes.
We say a representation function has linear transferability from the \emph{source domain} $\source$ to the \emph{target domain} $\target$ if a linear head trained on labeled data from $\source$ can accurately predict the class labels on $\target$.  E.g., the representations in Fig.~\ref{figure:main_figure} (middle) has linear transferability because the max-margin linear classifier trained on $S_\text{dog}$ vs. $S_\text{cat}$ also works well on $T_\text{dog}$ vs. $T_{cat}$. We note that linear separability is a different, weaker notion, which only requires the four groups of representations to be linearly separable from each other. 

Mathematically, we assume that the source domain and target domain are formed by $r$ clusters among $C_1,\dots, C_m$ for $r \le m/2$. Without loss of generality, assume that the source domain consists of cluster $S_1=C_1, \dots, S_r = C_r$ and the target domain consists of $T_1=C_{r+1}, \dots, T_{r} = C_{2r}$.  
Thus, $\source = \cup_{i\in[\nclass]}\source_i$ and $\target=\cup_{i\in[\nclass]}\target_i$.  We assume that the correct label for data in $S_i$ and $T_i$ is the cluster identity $i$.  Contrastive representations are trained on (samples of) the entire population data (which includes all $C_i$'s). The linear head is trained on the source with labels, and tested on the target.



Our key assumption is that the source and target classes are related correspondingly in the sense that there are more same-class cross-domain connections (between $S_i$ and $T_i$) than cross-class cross-domain connections (between $S_i$ and $T_j$ with $i\neq j$), formalized below. 

\begin{assumption}[Relative expansion]\label{assumption:uda} 
Let $\rho \triangleq \min_{i\in [r]} \phimin(T_i, S_i)$ be the minimum min-expansions from $T_i$ to $S_i$. 
For some sufficiently large universal constant $c$ (e.g., $c=8$ works), we assume that $\rho \ge c \cdot \alpha^2$ and that
\begin{align}
\rho = \min_{i\in [r]} \phimin(T_i, S_i)\ge c \cdot \max_{i\neq j}\cdot \phimax(T_i, S_j )\label{eqn:11}
\end{align}
\end{assumption}
Intuitively, equation~\eqref{eqn:11} says that every vertex in $\target_i$ has more edges connected to $\source_i$ than to $\source_j$. 
The condition $\rho\gtrsim \alpha^2$ says that the min-expansion $\rho$ is bigger than the square of max-expansion $\alpha$. 
This is reasonable because $\alpha\ll 1$ and thus $\alpha^2 \ll \alpha$, and we consider the min-expansion $\rho$ and max-expansion $\alpha$ to be somewhat comparable. 
In Section~\ref{section:mainresults_extension} we will relax this assumption and study the case when the average expansion $\phi(T_i, S_i)$ is larger than $\phi(T_i, S_j)$.

Our assumption is weaker than that in the prior work~\citep{cai2021theory} which also assumes expansion from $\source_i$ to $\target_i$ (though their goal is to study label propagation rather than contrastive learning). They assume the same-class cross-domain conductance $\phi(T_i, S_i)$ to be larger than the \textit{cross-class} same-domain conductance $\phi(S_i, S_j)$. Such an assumption limits the application to situations where the domains are far away from each other (such as DomainNet~\citep{peng2019moment}). 

Moreover, consider an interesting scenario with four clusters: photo dog, photo cat, sketch dog, and sketch cat. ~\cite{shen2021does} empirically showed that transferability can occur in the following two settings: (a) we view photo and sketch as domains: the source domain is photo dog vs photo cat, and the target domain is sketch dog  vs sketch cat; 
(b) we view cat and dog as domains, whereas  photo and sketch are classes: the source domain is photo dog vs sketch dog, and the target is photo cat vs sketch cat. The condition that cross-domain expansion is larger than cross-class expansion will fail to explain the transferability for one of these settings---if $\phi(\textup{photo dog}, \textup{sketch dog}) < \phi(\textup{photo dog}, \textup{photo cat})$, then it cannot explain (a), whereas if $\phi(\textup{photo dog}, \textup{sketch dog}) > \phi(\textup{photo dog}, \textup{photo cat})$, it cannot explain (b). 
In contrast, our assumption only requires conditions such as $\phi(\textup{photo dog}, \textup{sketch dog}) > \phi(\textup{photo dog}, \textup{sketch cat}) $, hence works for both settings. 

We will propose a simple and novel linear head that enables linear transferability.
Let $\psource$ be the data distribution restricted to the source domain.\footnote{Formally, we have $\psource(x) := \frac{\ww{x}}{\ww{\source}} \cdot \id{x\in\source}$, and $\ptarget(x)$ is defined similarly.} For $i\in[\nclass]$, we construct the following average representation for class $i$ in the source:\footnote{We assume access to independent samples from $\psource$ and thus $\head_i$ can be accurately estimated with finite labeled samples in the source domain.}
\begin{align}
	\head_{i} = \Exp{x\sim\psource}{\id{x\in\source_i}\cdot f(x)} \in\Real^k.
\end{align}
One of the most natural linear head is to use the average feature $b_i$'s as the weight vector for class $i$, as in many practical few shot learning algorithms~\citep{snell2017prototypical}.\footnote{We note that few-shot learning algorithms do not necessarily consider domain shift settings.} That is, we predict
\begin{align}
\pred(x) = \argmax_{i\in[\nclass]} \left\langle f(x), \head_{i}\right\rangle.\label{eqn:9}
\end{align}
This classifier can transfer to the target under relatively strong assumptions (see the special cases in the proof sketch in Section~\ref{section:sketch}), but is vulnerable to complex asymmetric structures in the graph. To strengthen the result, we consider a variant of this classifier with a proper preconditioning. 

To do so, we first define the representation covariance matrix which will play an important role:
$
\precond = \E{x\sim\pdata}[f(x)f(x)^\top].
$
The computation of this matrix only uses unlabeled data. Since $\precond\in\Real^{k\times k}$ is a low-dimensional matrix for $k$ not too large, we can accurately estimate it using finite samples from $\pdata$. For the ease of theoretical analysis, we assume that we can compute this matrix exactly.
Now we define a family of linear heads on the target domain: for $t\in\integer$, define
\begin{align}
	\pred_{t}(x) = \argmax_{i\in[\nclass]} \left\langle f(x), \precond^{t-1} \head_{i}\right\rangle.\label{eqn:13}
\end{align}
The case when $t=1$ corresponds to the linear head in equation~\eqref{eqn:9}. When $t$ is large, $g_t$ will care more about the correlation between $f(x)$ and $\head_i$ in those directions where the representation variance is large. 
Intuitively, directions with larger variance tend to contain information also in a more robust way, hence the preconditioner has a ``de-noising'' effect. See Section~\ref{section:sketch} for more on why the preconditioning improve the target error. Algorithm~\ref{algorithm:1} gives the pseudocode for this linear classification algorithm.

\begin{algorithm}[htb]\caption{Preconditioned feature averaging (PFA)}\label{algorithm:1}
	\begin{algorithmic}[1]
		\Require{Pre-trained representation extractor $f$, unlabeled data $\pdata$, source domain labeled data $\psource$, target domain test data $\tilde{x}$, integer $t\in\integer$}		
		\State Compute the preconditioner matrix 
		$
		\Sigma := \Exp{x\sim\pdata}{f(x)f(x)^\top}.
		$
		\For{every class $i\in[\nclass]$}
		\State Compute the mean feature of the class $i$: 
		$
		b_i := \Exp{(x,y)\sim\psource}{\id{y=i}\cdot f(x)}.
		$
		\EndFor		
		\State \Return prediction 
		$
		\argmax_{i\in[\nclass]} \left\langle f(x), \precond^{t-1} \head_{i}\right\rangle.
		$
	\end{algorithmic}
\end{algorithm}

We note that this linear head is different from prior work~\citep{shen2021does} where the linear head is trained with logistic loss. We made this modification since this head is more amenable to theoretical analysis.
In Section~\ref{sec:experiments} we show that this linear head also achieves superior empirical performance.

The error of a head $g$ on the target domain is defined as:
$
	\Err_\target(g) = \Exp{x\sim\ptarget}{\id{x\notin\target_{\pred(x)}}}.
$
The following theorem (proved in Appendix~\ref{appendix:uda}) shows that the linear head $g_t$ achieves high accuracy on the target domain with a properly chosen $t$:
\begin{theorem}\label{theorem:uda}
	Suppose that Assumption
	\ref{assumption:separation} and \ref{assumption:uda} holds, $\pdata(\source)/\pdata(\target) \le O(1)$.
	Let $f$ be a minimizer of the contrastive loss $\loss{2}(\cdot)$ and the head $g_t$ be defined in~\eqref{eqn:13}.
	Then, for any  $1\le t \le \rho/(8\alpha^2)$, we have
	$
			\Err_\target(g_t) \lesssim \frac{\nclass}{\alpha^2\lambda_{k+1}^2} \cdot \exp(-\frac{1}{2}t\lambda_{k+1}),
	$
where $\lambda_{k+1}$ is the $k$+$1$-th smallest eigenvalue of the Laplacian of the positive-pair graph.
Furthermore, suppose Assumption~\ref{assumption:intra_class_conductance} also holds and $k\ge 2m$, with $t=\rho/(8\alpha^2)$, we have
\begin{align}
\Err_\target(g_t) \lesssim \frac{\nclass}{\alpha^2\gamma^4} \cdot \exp\left(-\Omega\left(\frac{\rho\gamma^2}{\alpha^2}\right)\right).\label{eqn:15}
\end{align}
\end{theorem}

To see that RHS of equation~\eqref{eqn:15} implies small error, one can
consider a reasonable setting where the intra-cluster conductance is on the order of constants (i.e., $\gamma\ge\Omega(1)$). In this case, so long as $\rho\gg \alpha^2 \log(r/\alpha)$, we would have error bound $\Err_\target(g_t)\ll 1$. In general, as long as $\gamma \gg \alpha^{1/2}$ (the intra-cluster conductance is much larger than cross-cluster connections or its square root) and $\rho$ is comparable to $\alpha$, we have $\rho\gamma^2 \gg \alpha^2$ and thus a small upper bound of the error.  


Theorem~\ref{theorem:uda} shows that the error decreases as $t$ increases.  Intuitively, the PFA algorithm can be thought of as computing a low-rank approximation of a ``smoothed'' graph with normalized adjacency matrix $\norA^t$, where $\norA$ is the normalized adjacency matrix of the original positive-pair graph. A larger $t$ will make the low-rank approximation of $\norA^t$ more accurate, hence a smaller error. 
However, there's also an upper bound $t\le \rho/(8\alpha^2)$, since when $t$ is larger than this limit, the graph would be smoothed too much, hence the corresponding relationship in the graph between source and target classes would be erased. A more formal argument can be found in Section~\ref{section:sketch}.


We also note that our theorem allows ``overparameterization'' in the sense that a larger representation dimension $k$ always leads to a smaller error bound (since $\lambda_{k+1}$ is non-decreasing in $k$). Moreover, our theorem can be easily generalized to the setting where only polynomial samples of data are used to train the representations and the linear head,
assuming the realizability of the function class.

\subsection{Linear transferability with average relative expansion
}\label{section:mainresults_extension}
In this section, we relax Assumption~\ref{assumption:uda} and only assume that the \textit{total connections} from $\target_i$ to $\source_i$ is larger than that from $\target_i$ to $\source_j$, formalized below.

\begin{assumption}[Average relative expansion (weaker version of Assumption~\ref{assumption:uda})]\label{assumption:uda_multistep_interclass}
	For some sufficiently large $\tau>0$, we assume that 
	\begin{align}
		\forall i,  ~~ \phi(T_i, S_i) \ge \tau \cdot \alpha^2 ~~~~\textup{   and    }  ~~~\forall i\neq j,  ~~ \phi(T_i, S_i) \ge \tau \cdot \phi(T_i, S_j )
	\end{align}
\end{assumption}


The following theorem (proved in Appendix~\ref{appendix:uda_multistep}) generalizes Theorem~\ref{theorem:uda} in this setting. 
\begin{theorem}\label{theorem:uda_multistep}
	Suppose Assumptions \ref{assumption:separation}, \ref{assumption:intra_class_conductance} and\ref{assumption:uda_multistep_interclass}  hold, $\pdata(\source)/\pdata(\target) \le O(1)$, and feature dimension $k\ge 2m$.
	Then, for some $t=\Omega\left(\frac{1}{\gamma^2}\cdot \log\left(\frac{1}{\alpha}\right)\right)$, we have
\begin{align}
		\Err_\target(g_t) \lesssim\frac{\nclass}{\tau\gamma^8} \cdot \log^2\big(\frac{1}{\alpha}\big).
\end{align}
\end{theorem}


Again, consider a reasonable setting where the intra-cluster conductance is on the order of constants (i.e., $\gamma\ge\Omega(1)$). In this case, so long as $\tau$, the gap between same-class cross-domain connection and cross-class cross-domain connection is sufficiently large (e.g., $\tau \gg \log^2(r/\alpha)$), we would have an error bound $\Err_\target(g_t)\ll 1$. 

We note that the intra-cluster connections (Assumption~\ref{assumption:intra_class_conductance}) are necessary, when we only use the average relative expansion (Assumption~\ref{assumption:uda_multistep_interclass} as opposed to Assumption~\ref{assumption:uda}). Otherwise, there may exist subset $\tilde{\target}_i\subset \target_i$ that is completely disconnected from $\data\backslash\tilde{\target}_i$, hence no linear head trained on the source can be accurate on $\tilde{\target}_i$.


\newcommand{\tildef}{\tilde{f}}
\newcommand{\dhat}[1]{\hat{\hat{#1}}}
\section{Proof Sketch}\label{section:sketch}
{\bf Key challenge:} The analysis will involve careful understanding of how the spectrum of the normalized adjacency matrix of the positive-pair graph is influenced by three types of connections: (i) intra-cluster connections; (ii) connections between same-class cross-domain clusters (between $S_i$ and $T_i$), and (iii) connections between cross-class and cross-domain clusters (between $S_i$ and $T_j$ for $i\neq j$). Type (i) connections have the dominating contribution to the spectrum of the graph, contributing to the top eigenvalues. When analyzing the linear separability of the representations of the clusters, ~\citet{haochen2021provable} essentially show that type (ii) and (iii) are negligible compared to type (i) connections. However, this paper focuses on the linear transferability, where we need to compare how type (ii) and type (iii) connections influence the spectrum of the normalized adjancency matrix. However, such a comparison is challenging because they are both low-order terms compared to type (i) connections. Essentially, we develop a technique that can take out the influence of the type (i) connections so that they don't negatively influence our comparisons between type (ii) and type (iii) connections. 

Below we give a proof sketch of a sligthly weaker version of Theorem~\ref{theorem:uda} under a simplified setting.
First, we assume $r=2$, that is, there are two source classes $S_1$ and $S_2$, and two target classes $T_1$ and $T_2$. Second, we assume the marginal distribution over $x$ is uniform, that is, $w(x) = 1/N$ as this case typically capture the gist of the problem in spectral graph theory. 
Third, we will consider the simpler case where the normalized adjacency matrix $\norA$ is PSD, and the regularization strength $\reg=1$. 

Let $\tilde{f}(x)=\sqrt{\ww{x}}\cdot f(x)$ and $\tildeF\in\Real^{\Ndata\times k}$ be the matrix with $\tilde{f}(x)$ on its $x$-th row. \citet{haochen2021provable} (or Proposition~\ref{proposition:rank_k_approximation}) showed that matrix $\tildeF\tildeF^\top$ contains the top-$k$ eigenvectors of $\norA$. We will first give a proof for the case where $\tildeF\tildeF^\top$ exactly (Section~\ref{sec:warmup}) or near exactly  (Section~\ref{sec:almost_recover}) recovers $\norA$. Then we'll give a proof for the more realistic case where $\tildeF\tildeF^\top$ is not guaranteed to approximate $\norA$ accurately (Section~\ref{sec:realistic}).

\subsection{Warmup case: when $k=\infty$ and $\tildeF\tildeF^\top = \norA$} \label{sec:warmup}
In this extremely simplified setting, the inner product between the embeddings perfectly represents the graph (that is, $\langle\tildef(x), \tildef(x')\rangle = \norA_{x,x'}$). As a result, the connections between subsets of vertices, a graph quantity, can be written as a linear algebraic quantity involving $\tildeF$:
\begin{align}
	\ww{A, B} & =  \frac{1}{N}\cdot \one_{A}^\top \norA \one_{B} = \frac{1}{N}\cdot  \one_{A}^\top \tildeF\tildeF^\top \one_{B}\label{eqn:2}
\end{align}
where $\one_A\in \{0,1\}^N$ is the indicator vector for the set $A$,\footnote{Formally, we have $({\one_A})_x = 1$ iff $x\in A$.} and we used the assumption $\ww{x} = 1/N$. 

We start by considering the simple linear classifier which computes the difference between the means of the representations in two clusters.
\begin{align}
	v = \Exp{x\sim S_1}{f(x)} -  \Exp{x\sim S_2}{f(x)} = \tildeF^\top (\one_{S_1} - \one_{S_2}) \in \Real^k
\end{align}
This classifier corresponds to the head $g_1$ defined in Section~\ref{section:main_results},\footnote{Here because of the binary setting, the classifier can only involve one weight vector $v$ in $\Real^d$; this is equivalent to using two linear heads and then compute the maximum as in equation~\eqref{eqn:9}.} which suffices for the special case when $\tildeF\tildeF^\top = \norA$. Applying $v$ to any data point $x\in T_1\cup T_2$ results in the output $\hat{y}(x) = f(x)^\top v$. For notational simplicity, we consider $\dhat{y}(x) = \tildef(x)^\top v = \sqrt{w(x)}f(x)^\top \tildeF^\top (\one_{S_1} - \one_{S_2})$. Because $\hat{y}(x)$ and $\dhat{y}(x)$ has the same sign, it suffice to show that $\dhat{y}(x) > 0$ for $x\in T_1$ and $\dhat{y}(x) < 0$ for $x\in T_2$. Using equation~\eqref{eqn:2} that links the linear algebraic quantity to the graph quantity, 
\begin{align}
	\dhat{y}(x) & = \one_{x}^\top\tildeF\tildeF^\top (\one_{S_1} - \one_{S_2}) =\one_{x}^\top\norA (\one_{S_1} - \one_{S_2})  = N\cdot\left(w(x, S_1) - w(x, S_2)\right) \label{eqn:10}
\end{align}
In other words, the output $\dhat{y}$ depends on the relative expansions from $x$ to $S_1$ and from $x$ to $S_2$. By Assumption~\ref{assumption:uda} or Assumption~\ref{assumption:uda_multistep_interclass}, we have that when $x\in T_1$, $x$ has more expansion to $S_1$ than $S_2$, and vice versa for $x\in T_2$. Formally, by Assumption~\ref{assumption:uda}, we have that 
\begin{align}
	\forall x\in T_1, ~\phi(x, S_1) \ge \rho \gtrsim \phi(x, S_2)\textup{  and  } \forall x\in T_2, ~\phi(x, S_2) \ge \rho \gtrsim \phi(x, S_1)
\end{align}
Because $\phi(x, S_i) = w(x,S_i)/w(x) = N\cdot w(x,S_i)$, we have for $x\in T_1$, $w(x,S_1) > w(x,S_2)$, and therefore by equation~\eqref{eqn:10}, $\dhat{y}(x) > 0$. Similary when $x\in T_2$, $\dhat{y}(x) < 0$.  

\subsection{When $k\ll N$ and $\norA$ is almost rank-$k$} \label{sec:almost_recover}
Assuming $k=\infty$ is unrealistic since in most cases the feature is low-dimensional, i.e., $k\ll N$. However, so long as $\norA$ is almost rank-$k$, the above argument still works with minor modification.
More concretely, suppose $\norA$'s ($k$+1)-th largest eigenvalue, $1-\lambda_{k+1}$, is less than $\epsilon$. Then we have $\|\norA - \tildeF\tildeF^\top\|_{\textup{op}} = 1-\lambda_{k+1} \le \epsilon$. It turns out that when $\epsilon \ll 1$, we can straightforwardly adapt the proofs for the warm-up case with an additional $\epsilon$ error in the final target performance. The error comes from second step of equation~\eqref{eqn:10}.

\subsection{When $\norA$ is far from low-rank} \label{sec:realistic}
Unfortunately, a realistic graph's $\lambda_{k+1}$ is typically not close to 1 when $k\ll N$ (unless there's very strong symmetry in the graph as those cases in~\citet{shen2021does}). We aim to solve the more realistic and interesting case where $\lambda_{k+1}$ is a relatively small constant, e.g., $1/3$ or inverse polynomial in $d$. 
The previous argument stops working because $\tildeF\tildeF^\top$ is a \textit{very noisy} approximation of $\norA$: the error $\|\norA - \tildeF\tildeF^\top\|_{\textup{op}}= 1-\lambda_{k+1}$ is non-negligible and can be larger than $\|\tildeF\tildeF^\top\|_{\textup{op}} = \lambda_k$.
Our main approach is considering the power of $\norA$, which reduces the negative impact of smaller eigenvalues. 
Concretely, though $\|\norA - \tildeF\tildeF^\top\|_{\textup{op}}= 1-\lambda_{k+1}$ is non-negligible, $(\tildeF\tildeF^\top)^t$ is a much better approximation of $\norA^t$: 
\begin{align}
	\|\norA^t - (\tildeF\tildeF^\top)^t\|_{\textup{op}} = (1-\lambda_{k+1})^t = \epsilon \label{eqn:3}
\end{align}
when $t \ge \Omega(\log(1/\epsilon))$. Inspired by this, 
we consider the transformed linear classifier 
$
	v' = \Sigma^{t-1}\tildeF^\top (\one_{S_1} - \one_{S_2})
$,
where $\Sigma = \tildeF^\top \tildeF$ is the covariance matrix of the representations. 
Intuitively, multiplying $\Sigma$ forces the linear head to pay more attention to those large-variance directions of the representations, which are potentially more robust.
The classifier outputs the following on a target datapoint $x$ (with a rescaling of $\sqrt{w(x)}$ for convenience)
\begin{align}
	\dhat{y}'(x) & = \sqrt{w(x)} f(x)^\top v = \one_x^\top \tildeF\Sigma^{t-1}\tildeF^t (\one_{S_1} - \one_{S_2}) \nonumber\\
	&= \one_x^\top  (\tildeF\tildeF^\top)^t(\one_{S_1} - \one_{S_2})   \approx \one_x^\top  \norA^t (\one_{S_1} - \one_{S_2})
\end{align}
where the last step uses equation~\eqref{eqn:3}. Thus, to understand the sign of $\dhat{y}'(x)$, it suffices to compare $\one_x^\top  \norA^t \one_{S_1}$ with $\one_x^\top  \norA^t \one_{S_2}$. In other words, it suffices to prove that for $x\in T_1$,  $\one_x^\top  \norA^t \one_{S_1} > \one_x^\top  \norA^t \one_{S_2}$.

We control the quantity $\one_x^\top  \norA^t \one_{S_1}$ by leveraging the following connection between $\norA$ and a random walk on the graph. First, let $D = \textup{diag}(w)$ be the diagonal matrix with $D_{xx}=\ww{x}$, $A\in\Real^{\Ndata\times\Ndata}$ be the adjacency matrix, i.e., $A_{xx'}=\ww{x,x'}$.
Observe that $AD^{-1}$ is a transition matrix that defines a random walk on the graph, and $(AD^{-1})^t$ correspond to the transition matrix for $t$ steps of the random walk, denoted by $x_0, x_t, \dots, x_t$. Because $\norA^t = (D^{-1/2}AD^{-1/2})^t = D^{1/2}(D^{-1}A)^t D^{-1/2}$ and $D = 1/N\cdot \imatrix_{N\times N}$, we can verify that 
$
	\one_x^\top  \norA^t \one_{S_1} = \Pr[x_t\in S_1\mid x_0 = x]
$. That is, $\one_x^\top  \norA^t \one_{S_1}$ and $\one_x^\top  \norA^t \one_{S_2}$ are the probabilities to arrive at $S_1$ and $S_2$, respectively. form $x_0 =x$. 
Therefore, to prove that $\one_x^\top  \norA^t \one_{S_1}- \one_x^\top  \norA^t \one_{S_2} > 0$ for most $x\in T_1$, it suffices to prove 
that a $t$-step random walk starting from $T_1$ is more likely to arrive at $S_1$ than $S_2$. Intuitively, because $T_1$ has more connections to $S_1$ than $S_2$, hence a random walk starting from $T_1$ is more likely to arrive at $S_1$ than at $S_2$. In Section~\ref{appendix:uda}, we prove this by induction. 

\section{Simulations}\label{sec:experiments}
We empirically show that our proposed Algorithm~\ref{algorithm:1} achieves good performance on the unsupervised domain adaptation problem. 
We conduct experiments on BREEDS~\citep{santurkar2020breeds}---a dataset for evaluating unsupervised domain adaptation algorithms (where the source and target domains are constructed from ImageNet images). For pre-training, we run the spectral contrastive learning algorithm~\citep{haochen2021provable} on the joint set of source and target domain data. Unlike the previous convention of discarding the projection head, we use the output after projection MLP as representations, because we find that it significantly improves the performance (for models learned by spectral contrastive loss) and is more consistent with the theoretical formulation.
Given the pre-trained representations, we run Algorithm~\ref{algorithm:1} with different choices of $t$. For comparison, we use the linear probing baseline where we train a linear head with logistic regression on the source domain. The table below lists the test accuracy on the target domain for Living-17 and Entity-30---two datasets constructed by BREEDS. 
Additional details can be found in Section~\ref{section:experiment_detail}.

\begin{table}[!htb]
	\centering
	\small
	\tablestyle{6pt}{1.1}
	\begin{tabular}{c|ccc}
		& Linear probe & PFA (ours, $t=1$) & PFA (ours, $t=2$) \\
		\shline		
		Living-17 & 54.7 & 67.4 & 72.0 \\ 
		Entity-30 & 46.4 & 62.3 & 65.1
	\end{tabular}
	\vspace{.5em}
		\label{table:1}
	
\end{table}

Our experiments show that Algorithm~\ref{algorithm:1} achieves better domain adaptation performance than linear probing given the pre-trained representations. When $t=1$, our algorithm is simply computing the mean features of each class in the source domain, and then using them as the weight of a linear classifier. Despite having a lower accuracy than linear probing on the source domain (see section~\ref{section:experiment_detail} for the source domain accuracy), this simple algorithm achieves much higher accuracy on the target domain. When $t=2$, our algorithm incorporates the additional preconditioner matrix into the linear classifier, which further improves the domain adaptation performance. We note that our results on Entity-30 is better than~\cite{shen2021does} who compare with many state-of-the-art unsupervised domain adaptation methods, suggesting the superior performance of our algorithm.


\bibliographystyle{plainnat}
\bibliography{all}

\begin{thebibliography}{50}
\providecommand{\natexlab}[1]{#1}
\providecommand{\url}[1]{\texttt{#1}}
\expandafter\ifx\csname urlstyle\endcsname\relax
  \providecommand{\doi}[1]{doi: #1}\else
  \providecommand{\doi}{doi: \begingroup \urlstyle{rm}\Url}\fi

\bibitem[Arora et~al.(2019)Arora, Khandeparkar, Khodak, Plevrakis, and
  Saunshi]{arora2019theoretical}
Sanjeev Arora, Hrishikesh Khandeparkar, Mikhail Khodak, Orestis Plevrakis, and
  Nikunj Saunshi.
\newblock A theoretical analysis of contrastive unsupervised representation
  learning.
\newblock In \emph{International Conference on Machine Learning}, 2019.

\bibitem[Ben-David et~al.(2020)Ben-David, Rabinovitz, and
  Reichart]{ben2020perl}
Eyal Ben-David, Carmel Rabinovitz, and Roi Reichart.
\newblock Perl: Pivot-based domain adaptation for pre-trained deep
  contextualized embedding models.
\newblock \emph{Transactions of the Association for Computational Linguistics},
  8:\penalty0 504--521, 2020.

\bibitem[Ben-David et~al.(2010)Ben-David, Blitzer, Crammer, Kulesza, Pereira,
  and Vaughan]{ben2010theory}
Shai Ben-David, John Blitzer, Koby Crammer, Alex Kulesza, Fernando Pereira, and
  Jennifer~Wortman Vaughan.
\newblock A theory of learning from different domains.
\newblock \emph{Machine learning}, 79\penalty0 (1-2):\penalty0 151--175, 2010.

\bibitem[Blitzer et~al.(2007)Blitzer, Dredze, and
  Pereira]{blitzer2007biographies}
John Blitzer, Mark Dredze, and Fernando Pereira.
\newblock Biographies, bollywood, boom-boxes and blenders: Domain adaptation
  for sentiment classification.
\newblock In \emph{Proceedings of the 45th annual meeting of the association of
  computational linguistics}, pages 440--447, 2007.

\bibitem[Bobkov et~al.(1997)]{bobkov1997isoperimetric}
Sergey~G Bobkov et~al.
\newblock An isoperimetric inequality on the discrete cube, and an elementary
  proof of the isoperimetric inequality in gauss space.
\newblock \emph{The Annals of Probability}, 25\penalty0 (1):\penalty0 206--214,
  1997.

\bibitem[Cai et~al.(2021)Cai, Gao, Lee, and Lei]{cai2021theory}
Tianle Cai, Ruiqi Gao, Jason Lee, and Qi~Lei.
\newblock A theory of label propagation for subpopulation shift.
\newblock In \emph{International Conference on Machine Learning}, pages
  1170--1182. PMLR, 2021.

\bibitem[Caron et~al.(2020)Caron, Misra, Mairal, Goyal, Bojanowski, and
  Joulin]{caron2020unsupervised}
Mathilde Caron, Ishan Misra, Julien Mairal, Priya Goyal, Piotr Bojanowski, and
  Armand Joulin.
\newblock Unsupervised learning of visual features by contrasting cluster
  assignments.
\newblock \emph{arXiv preprint arXiv:2006.09882}, 33:\penalty0 9912--9924,
  2020.

\bibitem[Chen et~al.(2012)Chen, Xu, Weinberger, and Sha]{chen2012marginalized}
Minmin Chen, Zhixiang Xu, Kilian~Q Weinberger, and Fei Sha.
\newblock Marginalized denoising autoencoders for domain adaptation.
\newblock In \emph{Proceedings of the 29th International Coference on
  International Conference on Machine Learning}, pages 1627--1634, 2012.

\bibitem[Chen et~al.(2020{\natexlab{a}})Chen, Kornblith, Norouzi, and
  Hinton]{chen2020simple}
Ting Chen, Simon Kornblith, Mohammad Norouzi, and Geoffrey Hinton.
\newblock A simple framework for contrastive learning of visual
  representations.
\newblock In \emph{International conference on machine learning}, volume 119 of
  \emph{Proceedings of Machine Learning Research}, pages 1597--1607. PMLR,
  PMLR, 13--18 Jul 2020{\natexlab{a}}.

\bibitem[Chen et~al.(2020{\natexlab{b}})Chen, Kornblith, Swersky, Norouzi, and
  Hinton]{chen2020big}
Ting Chen, Simon Kornblith, Kevin Swersky, Mohammad Norouzi, and Geoffrey
  Hinton.
\newblock Big self-supervised models are strong semi-supervised learners.
\newblock \emph{arXiv preprint arXiv:2006.10029}, 2020{\natexlab{b}}.

\bibitem[Chen and He(2020)]{chen2020exploring}
Xinlei Chen and Kaiming He.
\newblock Exploring simple siamese representation learning.
\newblock \emph{arXiv preprint arXiv:2011.10566}, pages 15750--15758, June
  2020.

\bibitem[Chen et~al.(2020{\natexlab{c}})Chen, Fan, Girshick, and
  He]{chen2020improved}
Xinlei Chen, Haoqi Fan, Ross Girshick, and Kaiming He.
\newblock Improved baselines with momentum contrastive learning.
\newblock \emph{arXiv preprint arXiv:2003.04297}, 2020{\natexlab{c}}.

\bibitem[Chen et~al.(2020{\natexlab{d}})Chen, Wei, Kumar, and
  Ma]{chen2020selftraining}
Yining Chen, Colin Wei, Ananya Kumar, and Tengyu Ma.
\newblock Self-training avoids using spurious features under domain shift.
\newblock In \emph{Advances in Neural Information Processing Systems
  (NeurIPS)}, 2020{\natexlab{d}}.

\bibitem[Chung and Graham(1997)]{chung1997spectral}
Fan~RK Chung and Fan~Chung Graham.
\newblock \emph{Spectral graph theory}.
\newblock Number~92. American Mathematical Soc., 1997.

\bibitem[Gao et~al.(2021)Gao, Yao, and Chen]{gao2021simcse}
Tianyu Gao, Xingcheng Yao, and Danqi Chen.
\newblock Simcse: Simple contrastive learning of sentence embeddings.
\newblock \emph{arXiv preprint arXiv:2104.08821}, 2021.

\bibitem[Gretton et~al.(2008)Gretton, Smola, Huang, Schmittfull, Borgwardt, and
  Sch{\"o}lkopf]{gretton2008covariate}
Arthur Gretton, Alex Smola, Jiayuan Huang, Marcel Schmittfull, Karsten
  Borgwardt, and Bernhard Sch{\"o}lkopf.
\newblock Covariate shift by kernel mean matching.
\newblock In \emph{Dataset Shift in Machine Learning}. 2008.

\bibitem[HaoChen et~al.(2021)HaoChen, Wei, Gaidon, and Ma]{haochen2021provable}
Jeff~Z. HaoChen, Colin Wei, Adrien Gaidon, and Tengyu Ma.
\newblock Provable guarantees for self-supervised deep learning with spectral
  contrastive loss, 2021.

\bibitem[He et~al.(2020)He, Fan, Wu, Xie, and Girshick]{he2020momentum}
Kaiming He, Haoqi Fan, Yuxin Wu, Saining Xie, and Ross Girshick.
\newblock Momentum contrast for unsupervised visual representation learning.
\newblock In \emph{Proceedings of the IEEE/CVF Conference on Computer Vision
  and Pattern Recognition}, pages 9729--9738, June 2020.

\bibitem[Hendrycks et~al.(2020)Hendrycks, Liu, Wallace, Dziedzic, Krishnan, and
  Song]{hendrycks2020pretrained}
Dan Hendrycks, Xiaoyuan Liu, Eric Wallace, Adam Dziedzic, Rishabh Krishnan, and
  Dawn Song.
\newblock Pretrained transformers improve out-of-distribution robustness.
\newblock In \emph{Proceedings of the 58th Annual Meeting of the Association
  for Computational Linguistics}, pages 2744--2751, 2020.

\bibitem[Huang et~al.(2006)Huang, Gretton, Borgwardt, Sch{\"o}lkopf, and
  Smola]{huang2006correcting}
Jiayuan Huang, Arthur Gretton, Karsten~M Borgwardt, Bernhard Sch{\"o}lkopf, and
  Alex~J Smola.
\newblock Correcting sample selection bias by unlabeled data.
\newblock In \emph{Advances in neural information processing systems}, pages
  601--608, 2006.

\bibitem[Jean et~al.(2016)Jean, Burke, Xie, Davis, Lobell, and
  Ermon]{jean2016combining}
Neal Jean, Marshall Burke, Michael Xie, W.~Matthew Davis, David~B. Lobell, and
  Stefano Ermon.
\newblock Combining satellite imagery and machine learning to predict poverty.
\newblock \emph{Science}, 353, 2016.

\bibitem[Kim et~al.(2022)Kim, Wang, Sclaroff, and Saenko]{kim2022broad}
Donghyun Kim, Kaihong Wang, Stan Sclaroff, and Kate Saenko.
\newblock A broad study of pre-training for domain generalization and
  adaptation.
\newblock \emph{arXiv preprint arXiv:2203.11819}, 2022.

\bibitem[Kumar et~al.(2020)Kumar, Ma, and Liang]{kumar2020gradual}
Ananya Kumar, Tengyu Ma, and Percy Liang.
\newblock Understanding self-training for gradual domain adaptation.
\newblock In \emph{International Conference on Machine Learning (ICML)}, 2020.

\bibitem[Kumar et~al.(2022)Kumar, Raghunathan, Jones, Ma, and
  Liang]{kumar2022fine}
Ananya Kumar, Aditi Raghunathan, Robbie Jones, Tengyu Ma, and Percy Liang.
\newblock Fine-tuning can distort pretrained features and underperform
  out-of-distribution.
\newblock \emph{arXiv preprint arXiv:2202.10054}, 2022.

\bibitem[Lee et~al.(2014)Lee, Gharan, and Trevisan]{lee2014multiway}
James~R Lee, Shayan~Oveis Gharan, and Luca Trevisan.
\newblock Multiway spectral partitioning and higher-order cheeger inequalities.
\newblock \emph{Journal of the ACM (JACM)}, 61\penalty0 (6):\penalty0 1--30,
  2014.

\bibitem[Lee et~al.(2020)Lee, Lei, Saunshi, and Zhuo]{lee2020predicting}
Jason~D Lee, Qi~Lei, Nikunj Saunshi, and Jiacheng Zhuo.
\newblock Predicting what you already know helps: Provable self-supervised
  learning.
\newblock \emph{arXiv preprint arXiv:2008.01064}, 2020.

\bibitem[Louis and Makarychev(2014)]{louis2014approximation}
Anand Louis and Konstantin Makarychev.
\newblock Approximation algorithm for sparsest k-partitioning.
\newblock In \emph{Proceedings of the twenty-fifth annual ACM-SIAM symposium on
  Discrete algorithms}, pages 1244--1255. SIAM, 2014.

\bibitem[Mansour et~al.(2009)Mansour, Mohri, and
  Rostamizadeh]{mansour2009domain}
Yishay Mansour, Mehryar Mohri, and Afshin Rostamizadeh.
\newblock Domain adaptation: Learning bounds and algorithms.
\newblock \emph{arXiv preprint arXiv:0902.3430}, 2009.

\bibitem[Park et~al.(2020)Park, Lee, Yoo, Hur, and
  Yoon]{park2020jointcontrastive}
Changhwa Park, Jonghyun Lee, Jaeyoon Yoo, Minhoe Hur, and Sungroh Yoon.
\newblock Joint contrastive learning for unsupervised domain adaptation.
\newblock \emph{arXiv preprint arXiv:2006.10297}, 2020.

\bibitem[Peng et~al.(2019)Peng, Bai, Xia, Huang, Saenko, and
  Wang]{peng2019moment}
Xingchao Peng, Qinxun Bai, Xide Xia, Zijun Huang, Kate Saenko, and Bo~Wang.
\newblock Moment matching for multi-source domain adaptation.
\newblock In \emph{Proceedings of the IEEE International Conference on Computer
  Vision}, pages 1406--1415, 2019.

\bibitem[Sagawa et~al.(2022)Sagawa, Koh, Lee, Gao, Sang Michael~Xie, Kumar, Hu,
  Yasunaga, Henrik~Marklund, David, Stavness, Guo, Leskovec, Kate~Saenko,
  Levine, Finn, and Liang]{sagawa2021wilds}
Shiori Sagawa, Pang~Wei Koh, Tony Lee, Irena Gao, Kendrick~Shen Sang
  Michael~Xie, Ananya Kumar, Weihua Hu, Michihiro Yasunaga, Sara~Beery
  Henrik~Marklund, Etienne David, Ian Stavness, Wei Guo, Jure Leskovec,
  Tatsunori~Hashimoto Kate~Saenko, Sergey Levine, Chelsea Finn, and Percy
  Liang.
\newblock Extending the wilds benchmark for unsupervised adaptation.
\newblock In \emph{International Conference on Learning Representations}, 2022.

\bibitem[Santurkar et~al.(2020)Santurkar, Tsipras, and
  Madry]{santurkar2020breeds}
Shibani Santurkar, Dimitris Tsipras, and Aleksander Madry.
\newblock Breeds: Benchmarks for subpopulation shift.
\newblock \emph{arXiv}, 2020.

\bibitem[Shen et~al.(2022)Shen, Jones, Kumar, Xie, HaoChen, Ma, and
  Liang]{shen2021does}
Kendrick Shen, Robbie Jones, Ananya Kumar, Sang~Michael Xie, Jeff~Z. HaoChen,
  Tengyu Ma, and Percy Liang.
\newblock Connect, not collapse: Explaining contrastive learning for
  unsupervised domain adaptation.
\newblock \emph{arXiv preprint arXiv:2204.00570}, 2022.

\bibitem[Shimodaira(2000)]{shimodaira2000improving}
Hidetoshi Shimodaira.
\newblock Improving predictive inference under covariate shift by weighting the
  log-likelihood function.
\newblock \emph{Journal of statistical planning and inference}, 90\penalty0
  (2):\penalty0 227--244, 2000.

\bibitem[Snell et~al.(2017)Snell, Swersky, and Zemel]{snell2017prototypical}
Jake Snell, Kevin Swersky, and Richard Zemel.
\newblock Prototypical networks for few-shot learning.
\newblock \emph{Advances in neural information processing systems}, 30, 2017.

\bibitem[Su et~al.(2021)Su, Liu, Meng, Lan, Shu, Shareghi, and
  Collier]{su2021tacl}
Yixuan Su, Fangyu Liu, Zaiqiao Meng, Tian Lan, Lei Shu, Ehsan Shareghi, and
  Nigel Collier.
\newblock Tacl: Improving bert pre-training with token-aware contrastive
  learning, 2021.

\bibitem[Sugiyama et~al.(2007)Sugiyama, Krauledat, and
  M{\~A}{\v{z}}ller]{sugiyama2007covariate}
Masashi Sugiyama, Matthias Krauledat, and Klaus-Robert M{\~A}{\v{z}}ller.
\newblock Covariate shift adaptation by importance weighted cross validation.
\newblock \emph{Journal of Machine Learning Research}, 8\penalty0
  (May):\penalty0 985--1005, 2007.

\bibitem[Thota and Leontidis(2021)]{thota2021contrastive}
Mamatha Thota and Georgios Leontidis.
\newblock Contrastive domain adaptation.
\newblock In \emph{Proceedings of the IEEE/CVF Conference on Computer Vision
  and Pattern Recognition}, pages 2209--2218, 2021.

\bibitem[Tosh et~al.(2020)Tosh, Krishnamurthy, and Hsu]{tosh2020contrastive}
Christopher Tosh, Akshay Krishnamurthy, and Daniel Hsu.
\newblock Contrastive estimation reveals topic posterior information to linear
  models.
\newblock \emph{arXiv:2003.02234}, 2020.

\bibitem[Tosh et~al.(2021)Tosh, Krishnamurthy, and Hsu]{tosh2021contrastive}
Christopher Tosh, Akshay Krishnamurthy, and Daniel Hsu.
\newblock Contrastive learning, multi-view redundancy, and linear models.
\newblock In \emph{Algorithmic Learning Theory}, pages 1179--1206. PMLR, 2021.

\bibitem[Wang et~al.(2021)Wang, Wu, Weng, Chen, Qi, and
  Jiang]{wang2021crossdomain}
Rui Wang, Zuxuan Wu, Zejia Weng, Jingjing Chen, Guo-Jun Qi, and Yu-Gang Jiang.
\newblock Cross-domain contrastive learning for unsupervised domain adaptation.
\newblock \emph{arXiv preprint arXiv:2106.05528}, 2021.

\bibitem[Wei et~al.(2020)Wei, Shen, Chen, and Ma]{wei2020theoretical}
Colin Wei, Kendrick Shen, Yining Chen, and Tengyu Ma.
\newblock Theoretical analysis of self-training with deep networks on unlabeled
  data, 2020.
\newblock URL \url{https://openreview.net/forum?id=rC8sJ4i6kaH}.

\bibitem[Wortsman et~al.(2021)Wortsman, Ilharco, Li, Kim, Hajishirzi, Farhadi,
  Namkoong, and Schmidt]{wortsman2021robust}
Mitchell Wortsman, Gabriel Ilharco, Mike Li, Jong~Wook Kim, Hannaneh
  Hajishirzi, Ali Farhadi, Hongseok Namkoong, and Ludwig Schmidt.
\newblock Robust fine-tuning of zero-shot models.
\newblock \emph{arXiv preprint arXiv:2109.01903}, 2021.

\bibitem[Wortsman et~al.(2022)Wortsman, Ilharco, Gadre, Roelofs, Gontijo-Lopes,
  Morcos, Namkoong, Farhadi, Carmon, Kornblith, et~al.]{wortsman2022model}
Mitchell Wortsman, Gabriel Ilharco, Samir~Yitzhak Gadre, Rebecca Roelofs,
  Raphael Gontijo-Lopes, Ari~S Morcos, Hongseok Namkoong, Ali Farhadi, Yair
  Carmon, Simon Kornblith, et~al.
\newblock Model soups: averaging weights of multiple fine-tuned models improves
  accuracy without increasing inference time.
\newblock \emph{arXiv preprint arXiv:2203.05482}, 2022.

\bibitem[Xie et~al.(2020)Xie, Kumar, Jones, Khani, Ma, and Liang]{xie2020n}
Sang~Michael Xie, Ananya Kumar, Robbie Jones, Fereshte Khani, Tengyu Ma, and
  Percy Liang.
\newblock In-n-out: Pre-training and self-training using auxiliary information
  for out-of-distribution robustness.
\newblock In \emph{International Conference on Learning Representations}, 2020.

\bibitem[Zbontar et~al.(2021)Zbontar, Jing, Misra, LeCun, and
  Deny]{zbontar2021barlow}
Jure Zbontar, Li~Jing, Ishan Misra, Yann LeCun, and St{\'e}phane Deny.
\newblock Barlow twins: Self-supervised learning via redundancy reduction.
\newblock \emph{arXiv preprint arXiv:2103.03230}, 2021.

\bibitem[Zhang et~al.(2019)Zhang, Liu, Long, and Jordan]{zhang2019bridging}
Yuchen Zhang, Tianle Liu, Mingsheng Long, and Michael~I Jordan.
\newblock Bridging theory and algorithm for domain adaptation.
\newblock \emph{arXiv preprint arXiv:1904.05801}, pages 7404--7413, 2019.

\bibitem[Zhao et~al.(2019)Zhao, Combes, Zhang, and Gordon]{zhao2019learning}
Han Zhao, Remi Tachet~Des Combes, Kun Zhang, and Geoffrey Gordon.
\newblock On learning invariant representations for domain adaptation.
\newblock In \emph{Proceedings of the 36th International Conference on Machine
  Learning}, pages 7523--7532. PMLR, 09--15 Jun 2019.
\newblock URL \url{http://proceedings.mlr.press/v97/zhao19a.html}.

\bibitem[Ziser and Reichart(2017)]{ziser2017neural}
Yftah Ziser and Roi Reichart.
\newblock Neural structural correspondence learning for domain adaptation.
\newblock In \emph{Proceedings of the 21st Conference on Computational Natural
  Language Learning (CoNLL 2017)}, pages 400--410, 2017.

\bibitem[Ziser and Reichart(2018)]{ziser2018deep}
Yftah Ziser and Roi Reichart.
\newblock Deep pivot-based modeling for cross-language cross-domain transfer
  with minimal guidance.
\newblock In \emph{Proceedings of the 2018 Conference on Empirical Methods in
  Natural Language Processing}, pages 238--249, 2018.

\end{thebibliography}

	\section*{Checklist}

	The checklist follows the references.  Please
	read the checklist guidelines carefully for information on how to answer these
	questions.  For each question, change the default \answerTODO{} to \answerYes{},
	\answerNo{}, or \answerNA{}.  You are strongly encouraged to include a {\bf
		justification to your answer}, either by referencing the appropriate section of
	your paper or providing a brief inline description.  For example:
	\begin{itemize}
		\item Did you include the license to the code and datasets? \answerYes{See Section~\ref{gen_inst}.}
		\item Did you include the license to the code and datasets? \answerNo{The code and the data are proprietary.}
		\item Did you include the license to the code and datasets? \answerNA{}
	\end{itemize}
	Please do not modify the questions and only use the provided macros for your
	answers.  Note that the Checklist section does not count towards the page
	limit.  In your paper, please delete this instructions block and only keep the
	Checklist section heading above along with the questions/answers below.

	\begin{enumerate}

		\item For all authors...
		\begin{enumerate}
			\item Do the main claims made in the abstract and introduction accurately reflect the paper's contributions and scope?
			\answerYes{}
			\item Did you describe the limitations of your work?
			\answerYes{}
			\item Did you discuss any potential negative societal impacts of your work?
			\answerNA{}
			\item Have you read the ethics review guidelines and ensured that your paper conforms to them?
			\answerYes{}
		\end{enumerate}

		\item If you are including theoretical results...
		\begin{enumerate}
			\item Did you state the full set of assumptions of all theoretical results?
			\answerYes{}
			\item Did you include complete proofs of all theoretical results?
			\answerYes{}
		\end{enumerate}

		\item If you ran experiments...
		\begin{enumerate}
			\item Did you include the code, data, and instructions needed to reproduce the main experimental results (either in the supplemental material or as a URL)?
			\answerNo{}
			\item Did you specify all the training details (e.g., data splits, hyperparameters, how they were chosen)?
			\answerYes{}
			\item Did you report error bars (e.g., with respect to the random seed after running experiments multiple times)?
			\answerNo{}
			\item Did you include the total amount of compute and the type of resources used (e.g., type of GPUs, internal cluster, or cloud provider)?
			\answerNo{}
		\end{enumerate}

		\item If you are using existing assets (e.g., code, data, models) or curating/releasing new assets...
		\begin{enumerate}
			\item If your work uses existing assets, did you cite the creators?
			\answerYes{}
			\item Did you mention the license of the assets?
			\answerNA{}
			\item Did you include any new assets either in the supplemental material or as a URL?
			\answerNA{}
			\item Did you discuss whether and how consent was obtained from people whose data you're using/curating?
			\answerNA{}
			\item Did you discuss whether the data you are using/curating contains personally identifiable information or offensive content?
			\answerNA{}
		\end{enumerate}

		\item If you used crowdsourcing or conducted research with human subjects...
		\begin{enumerate}
			\item Did you include the full text of instructions given to participants and screenshots, if applicable?
			\answerNA{}
			\item Did you describe any potential participant risks, with links to Institutional Review Board (IRB) approvals, if applicable?
			\answerNA{}
			\item Did you include the estimated hourly wage paid to participants and the total amount spent on participant compensation?
			\answerNA{}
		\end{enumerate}

	\end{enumerate}


\newpage
\appendix

\section{Additional experiment details}\label{section:experiment_detail}
Unlike the previous convention of discarding the projection head and using the pre-MLP layers as the features~\cite{chen2020simple}, we use the final output of the neural nets as representations, because we find that it significantly improves the performance (for models learned by spectral contrastive loss) and is more consistent with the theoretical formulation.

For the architecture, we use ResNet50 followed by a 3-layer MLP projection head, where the hidden and output dimensions are 1024. For pre-training, we use the spectral contrastive learning algorithm~\cite{haochen2021provable} with hyperparameter $\mu=10$, and use the same augmentation strategy as described in~\cite{chen2020exploring}. We train the neural network using SGD with momentum 0.9. The learning rate starts at 0.05 and decreases to 0 with a cosine schedule. We use weight decay 0.0001 and train for 800 epochs with batch size 256. 

For linear probe experiments, we train a linear head using SGD with batch size 256 and weight decay 0 for 100 epochs, learning rate starts at 30.0 and is decayed by 10x at the 60th and 80th epochs. The classification accuracy on the source and target domains are listed in Table~\ref{table:2}:
\begin{table}[!htb]
	\centering
	\small
	\tablestyle{6pt}{1.1}
	\begin{tabular}{c|ccc}
		& linear probe & Ours (t=1) & Ours (t=2)  \\
		\shline		
		Living-17 & 91.3 / 54.7 & 92.6 / 67.4 & 90.5 / 72.0 \\ 
		Entity-30 & 84.8 / 46.4  & 82.8 / 62.3 & 77.3 / 65.1
	\end{tabular}
	\vspace{.5em}
	\caption{Accuracy (\%) of linear probing and Algorithm~\ref{algorithm:1} on the source and target domain. The number before and after slash are on the source and target domains, respectively. The numbers after slash are the same as in Table~\ref{table:1}.
		\label{table:2}
	}
\end{table}

\section{The generalized spectral contrastive loss}\label{appendix:generalized_loss}

Recall that the spectral contrastive loss~\cite{haochen2021provable} is defined as 
\begin{align}
\lossscl(f) = -2\cdot \Exp{(x, x^+)\sim \ppos}{ f(x)^\top f(x^+)} + \Exp{x,x' \sim \pdata}{(f(x)^\top f(x'))^2}
\end{align}

The following proposition shows that the generalized spectral contrastive loss $\loss{\reg}$ recovers the spectral contrastive loss when $\reg=1$.
\begin{proposition}\label{proposition:generalized_loss}
	For all $f:\data\rightarrow \Real^k$, we have
	\begin{align}
		\loss{1}(f) = \lossscl(f) + c, 
	\end{align}
	where $c$ does not depend on $f$. 
\end{proposition}

\begin{proof}[Proof of Proposition~\ref{proposition:generalized_loss}]
	Define matrix $\tildeF\in\Real^{\Ndata\times k}$ be such that the $x$-th row of it contains $\sqrt{\ww{x}}\cdot f(x)$. We have
	\begin{align}
		\loss{\reg}(f) &= \Exp{(x, x^+)\sim \ppos}{\norm{f(x)-f(x^+)}_2^2} + \reg \cdot \norm{\Exp{x\sim \pdata}{f(x)f(x)^\top} - \imatrix_k}_F^2\\
		&= \sum_{x, x'\in \data} \ww{x, x'} \norm{f(x)-f(x')}_2^2 + \reg \cdot \norm{\tildeF^\top\tildeF - \imatrix_k}_F^2\\
		&= 2\sum_{x\in\data} \ww{x} \norm{f(x)}_2^2 -2\sum_{x, x'\in \data} \ww{x, x'} f(x)^\top f(x') + \reg\cdot\Tr\left(\left(\tildeF^\top \tildeF - \imatrix_k\right)^2\right)\\
		&=  2\Tr\left(\tildeF\tildeF^\top\right) -2\Exp{(x, x^+)\sim\ppos}{f(x)^\top f(x^+)} + \reg \Tr\left(\left(\tildeF^\top\tildeF\right)^2\right) -2\reg\Tr\left(\tildeF^\top\tildeF\right) + \textup{const}.
	\end{align}
	When $\reg=1$, notice that $\Tr\left(\tildeF\tildeF^\top\right) = \Tr\left(\tildeF^\top\tildeF\right)$ and $\Tr\left(\left(\tildeF\tildeF^\top\right)^2\right) = \Tr\left(\left(\tildeF^\top\tildeF\right)^2\right)$, we have
	\begin{align}
		\loss{1}(f) &=  -2\Exp{(x, x^+)\sim\ppos}{f(x)^\top f(x^+)} + \Tr\left(\left(\tildeF\tildeF^\top\right)^2\right) +  \textup{const}\\
		&= 	-2\Exp{(x, x^+)\sim\ppos}{f(x)^\top f(x^+)} + \Exp{x, x'\sim\pdata}{\left(f(x)^\top f(x')\right)^2} + \textup{const}. \\
		&= \lossscl(f) + \textup{const}.
	\end{align}

\end{proof}

\section{Relationship between contrastive representations and spectral decomposition}\label{appendix:rank_k_approximation}

\citet{haochen2021provable} showed that minimizing spectral contrastive loss is equivalent to spectral clustering on the positive-pair graph. We introduce basic concepts in spectral graph theory and extend this result slightly to the generalized spectral contrastive loss. 
We call $\norA\in\Real^{\Ndata\times\Ndata}$ the \textit{normalized adjacency matrix} of $\graph(\data, w)$ if $\norA_{xx'}={\ww{x,x'}}/{\sqrt{\ww{x}\ww{x'}}}$.\footnote{We index $\norA$ by $(x, x')\in\data\times\data$. Generally, we will index the $\Ndata$-dimensional axis of an array by $x\in\data$.} 
Let $\laplacian := \imatrix_{\Ndata\times\Ndata} -\norA$ be the \textit{Laplacian} of $\graph(\data, w)$. It is well-known~\citep{chung1997spectral} that $\laplacian$ is a PSD matrix with all eigenvalues in $[0, 2]$. We use $\lambda_i$ to denote the $i$-th smallest eigenvalue of $\laplacian$.
For a symmetric matrix $M$, we say $M_{[k]}$ is the best rank-$k$ \textit{PSD} approximation of $M$ if it is a rank-$k$ \textit{PSD} matrix that minimizes $\norm{M_{[k]} - M}_F^2$.

Representations learned from $\loss{\reg}$ turn out to be closely related to the low-rank approximation of $\norA$, as shown in the following Proposition.
\begin{proposition}\label{proposition:rank_k_approximation}
	Let $f:\data\rightarrow\Real^k$ be a minimizer of $\loss{1}(\cdot)$, $F\in\Real^{\Ndata\times k}$ be the matrix where the $x$-th row contains $f(x)$, and $D = \textup{diag}(w)$ be the diagonal matrix with $D_{xx} = \ww{x}$. Then, we have 
	\begin{align}
		{D}^{1/2} FF^\top {D}^{1/2} =\norA_{[k]}.
	\end{align}
	
	More generally, when $f:\data\rightarrow\Real^k$ is a minimizer of $\loss{\reg}(\cdot)$, ${D}^{1/2} FF^\top {D}^{1/2}$ is the best rank-$k$ PSD approximation of
	$\frac{1}{\reg}\cdot \norA + (1-\frac{1}{\reg})\cdot \imatrix_{\Ndata\times\Ndata}$.
\end{proposition}

\begin{remark}
Proposition~\ref{proposition:rank_k_approximation} can be seen as a simple extension of Lemma 3.2 in \citet{haochen2021provable}, which correspond to the case when $\reg=1$. The extension is helpful because we will work with $\sigma>1$. E.g., we set $\sigma=2$ in Section~\ref{section:main_results}, which makes $\frac{1}{\reg}\cdot \norA + (1-\frac{1}{\reg})\cdot \imatrix_{\Ndata\times\Ndata}$ a \textit{PSD} matrix; hence its best rank-$k$ \textit{PSD} approximation is the same as best rank-$k$ approximation. 
\end{remark}

\begin{proof}[Proof of Proposition~\ref{proposition:rank_k_approximation}]
	Define $\tildeF := D^{\frac{1}{2}} F$.
	Following the Proof of Proposition~\ref{proposition:generalized_loss}, we have
	\begin{align}
		\loss{\reg}(f) =  2\Tr\left(\tildeF\tildeF^\top\right) -2\Exp{(x, x^+)\sim\ppos}{f(x)^\top f(x^+)} + \reg \Tr\left(\left(\tildeF^\top\tildeF\right)^2\right) -2\reg\Tr\left(\tildeF^\top\tildeF\right) + \textup{const}.
	\end{align}
	Notice that $\Tr\left(\tildeF\tildeF^\top\right) = \Tr\left(\tildeF^\top\tildeF\right)$ and $\Tr\left(\left(\tildeF\tildeF^\top\right)^2\right) = \Tr\left(\left(\tildeF^\top\tildeF\right)^2\right)$, we have
	\begin{align}
		\loss{\reg}(f) &= \reg \Tr\left(\left(\tildeF\tildeF^\top\right)^2\right) - 2\Tr \left(\left(\norA + (\reg-1)\imatrix_{\Ndata\times\Ndata}\right) \tildeF\tildeF^\top\right) + \textup{const}\\
		&= \reg \norm{\tildeF\tildeF^\top - \left(\frac{1}{\reg}\norA + (1-\frac{1}{\reg})\imatrix_{\Ndata\times\Ndata}\right)}_F^2 + \textup{const}.
	\end{align}
	Therefore, directly applying Eckart-Young-Mirsky theorem finishes the proof.
\end{proof}

\section{Improved bound on linear separability}\label{appendix:alpha2bound}

Let $f:\data\rightarrow\Real^k$ be a representation function with dimension $k>m$. 
For a matrix $\headmatrix\in\Real^{k\times\ncluster}$, we define the linear head as $g_{\headmatrix}(x) = \argmax_{i\in[\ncluster]} (\headmatrix^\top f(x))_i$. 
The \emph{linear probing error} of $f$ is the minimal possible error of using such a linear head to predict which cluster a datapoint belongs to: 
\begin{align}
	\Err(f) := \min_{\headmatrix\in\Real^{k\times\ncluster}} \Exp{x\sim\pdata}{\id{x\notin C_{g_{ \headmatrix}(x)}}}.
\end{align}
We say the representation $f$ has linear separability if the linear probing error is small.

\citet{haochen2021provable} prove the linear separability of spectral contrastive representations.
In particular, they prove that $\Err(f)\le O(\alpha/\lambda_{k+1})$ where $\lambda_{k+1}$ is the ($k$+1)-th smallest eigenvalue of the Laplacian. 
When $k$ is set to be large enough---larger than the total number of distinct semantic meanings in the graph---$\graph$ cannot be partitioned into $k$ disconnected clusters, hence $\lambda_{k+1}$ is big (e.g., on the order of constant) according to Cheeger's inequality, and we have $\Err(f)\le O(\alpha)$.\footnote{High-order Cheeger's inequality  establishes a precise connection between $\lambda_{k}$ and the clusterabilty of the graph. Loosely speaking, when the graph cannot be partition into $k/2$ pieces with expansion at most $\gamma$, then $\lambda_k \gtrsim \gamma^2$ (see~\cite{lee2014multiway,louis2014approximation}, c.f. Lemma B.4 of~\cite{haochen2021provable}.) }

The  lemma below shows that Assumption~\ref{assumption:separation} enables a better bound on the linear probing errors.

\begin{lemma}\label{theorem:alpha2bound}
	Suppose that Assumption~\ref{assumption:separation} holds.
	Let $f:\data\rightarrow\Real^k$ be a minimizer of the generalized spectral contrastive loss $\loss{\reg}(\cdot)$ for $\reg\ge\eigval_k$. 
	Then, the linear probing error satisfies 
	\begin{align}
		\Err(f) \lesssim {\ncluster\alpha^2}/{\eigval_{k+1}^2}.
	\end{align}
	where $\eigval_{k+1}$ is the $(k+1)$-th smallest eigenvalue of the Laplacian matrix of $\graph(\data, w)$.
\end{lemma}


\begin{remark}
Since the separation assumption inherently implies small $\lambda_\ncluster$ (according to Cheeger's inequality), one needs to choose the representation dimension $k> \ncluster$ for the bound to be non-vacuous.
When $\ncluster\le O(1)$ and $\lambda_{k+1}\ge\Omega(1)$,
Lemma~\ref{theorem:alpha2bound} implies that the linear probing error of $f$ is at most $O(\alpha^2)$, which improves upon the previous $O(\alpha)$ bound. 
\end{remark}

We first introduce the following claim, which controls the Rayleigh quotient for Laplacian square $\laplacian^2$ and the indicator vector of one cluster.
\begin{claim}\label{lemma:laplaciansquare}
	Suppose that Assumption~\ref{assumption:separation} holds.
	Let $i\in[\ncluster]$ be the index of one cluster. Let $\wvec_i\in\Real^\Ndata$ be a vector such that its $x$-th dimension is $\sqrt{\ww{x}}$ when ${x}\in C_i$, $0$ otherwise. Then, we have
	\begin{align}
		\wvec_i^\top \laplacian^2 \wvec_i \le 2\alpha^2 \norm{\wvec_i}_2^2.
	\end{align}
\end{claim}
\begin{proof}[Proof of Claim~\ref{lemma:laplaciansquare}]	
	We first bound every dimension of the vector $\laplacian \wvec_i = (\imatrix-\norA)\wvec_i$. Let $x\in C_i$, we have
	\begin{align}
		(\norA \wvec_i)_x &=\sum_{\tilde{x}\in C_i} \frac{\ww{x, \tilde{x}}}{\sqrt{\ww{x}}\sqrt{\tilde{x}}} \sqrt{\tilde{x}}\\
		&= \left(\sum_{\tilde{x}\in C_i} \ww{x, \tilde{x}}\right) \cdot \frac{1}{\sqrt{\ww{x}}}\\
		&
		\begin{cases}
			\ge \frac{1}{\sqrt{\ww{x}} (1-\alpha)} \cdot \sum_{\tilde{x}\in\data} \ww{x, \tilde{x}} = (1-\alpha) \sqrt{\ww{x}}. \\
			\le \frac{1}{\sqrt{\ww{x}}} \cdot \sum_{\tilde{x}\in\data} \ww{x, \tilde{x}} = \sqrt{\ww{x}}.
		\end{cases}
	\end{align}
	Let $x'\notin C_i$, we have
	\begin{align}
		(\norA \wvec_i)_{x'} &= \sum_{\tilde{x}\in C_i} \frac{\ww{x', \tilde{x}}}{\sqrt{\ww{x'}}\sqrt{\ww{\tilde{x}}}} \cdot \sqrt{\ww{\tilde{x}}}\\
		&= \frac{1}{\sqrt{\ww{x'}}} \cdot \sum_{\tilde{x}\in C_i}\ww{x', \tilde{x}} \\
		&\begin{cases}
			\le \alpha \sqrt{\ww{x'}}\\
			\ge 0.
		\end{cases}
	\end{align}
	Therefore, we have $((\imatrix-\norA)\wvec_i)_x \in[0, \alpha\sqrt{\ww{x}}]$ for any $x\in C_i$, and $((\imatrix-\norA)\wvec_i)_{x'} \in[-\alpha\sqrt{\ww{x}}, 0]$ for any $x'\notin C_i$. Let $\wvec'_i \triangleq (\imatrix-\norA)\wvec_i$ as a shorthand,  we have
	\begin{align}
		\wvec_i^\top \norA \wvec'_i &= \sum_{\tilde{x}\in C_i, x\in C_i} \frac{\ww{w, \tilde{x}}}{\sqrt{\ww{\tilde{x}}}\sqrt{\ww{x}}} \cdot \sqrt{\ww{\tilde{x}}} \cdot (\wvec'_i)_x +  \sum_{\tilde{x}\in C_i, x'\notin C_i} \frac{\ww{\tilde{x}, x'}}{\sqrt{\ww{\tilde{x}}}\sqrt{{\ww{x'}}}}\cdot \sqrt{\ww{\tilde{x}}} \cdot (\wvec'_i)_{x'}\\
		&= \sum_{\tilde{x}\in C_i, x\in C_i} \frac{\ww{w, \tilde{x}}}{\sqrt{\ww{x}}} \cdot (\wvec'_i)_x +  \sum_{\tilde{x}\in C_i, x'\notin C_i} \frac{\ww{\tilde{x}, x'}}{\sqrt{{\ww{x'}}}}\cdot (\wvec'_i)_{x'}.
	\end{align}
	Also notice that
	\begin{align}
		\wvec_i^\top \imatrix \wvec'_i = \sum_{x\in C_i} \sqrt{\ww{x}}\cdot (\wvec'_i)_x.
	\end{align}
	Therefore, we have
	\begin{align}
		\wvec_i^\top (\imatrix-\norA)\wvec'_i = Q_1 + Q_2,
	\end{align}
	where 
	\begin{align}
		Q_1 &\triangleq \sum_{x\in C_i} \sqrt{\ww{x}}\cdot (\wvec'_i)_x - \sum_{\tilde{x}\in C_i, x\in C_i} \frac{\ww{\tilde{x}, x}}{\sqrt{\ww{x}}}\cdot (\wvec'_i)_x\\
		&= \sum_{x\in C_i} \left(\frac{\sum_{\tilde{x}\notin C_i} \ww{\tilde{x}, x}}{\sqrt{\ww{x}}}(\wvec'_i)_x\right) \in \left[0, \alpha^2 \sum_{x\in C_i} \ww{x}\right],
	\end{align}
	and
	\begin{align}
		Q_2 \triangleq -\sum_{\tilde{x}\in C_i, x'\notin C_i} \frac{\ww{\tilde{x}, x'}}{\sqrt{\ww{x'}}} (\wvec'_i)_{x'} \in \left[0, \alpha^2 \sum_{x\in C_i} \ww{x}\right].
	\end{align}
	As a result, we have
	\begin{align}
		\wvec_i^\top \laplacian^2 \wvec_i = \wvec_i^\top (\imatrix-\norA) \wvec'_i \le 2\alpha^2 \sum_{x\in C_i} \ww{x} = 2\alpha^2 \norm{\wvec_i}_2^2.
	\end{align}
\end{proof}

Now we use the above claim to prove Lemma~\ref{theorem:alpha2bound}.
\begin{proof}[Proof of Lemma~\ref{theorem:alpha2bound}]
	Define matrix $\tildeF\in\Real^{\Ndata\times k}$ be such that the $x$-th row of it contains $\sqrt{\ww{x}}\cdot f(x)$. According to Proposition~\ref{proposition:rank_k_approximation}, the column span of $\tildeF$ is exactly the span of the $k$ largest positive eigenvectors of $\frac{1}{\reg}\cdot \norA + (1-\frac{1}{\reg})\cdot \imatrix_{\Ndata\times\Ndata}$, hence is the span of the $k$ smallest eigenvectors of $\laplacian$. For every $i\in[\ncluster]$, define vector $\wvec_i\in\Real^\Ndata$ be a vector such that its $x$-th dimension is $\sqrt{\ww{x}}$ when $x\in C_i$, $0$ otherwise. Let vector $B_i\in\Real^k$ be such that $\tildeF B_i$ is the projection of $\wvec_i$ onto the span of the $k$ smallest eigenvectors of $\laplacian$. Let $B\in\Real^{k\times \ncluster}$ be the matrix where $B_i$ is the $i$-th column.
	
	For any $i\in[\ncluster]$, we have 
	\begin{align}
		\sum_{x\in\data} \ww{x} \left(B_i^\top f(x) - \id{\clusterf{x}=c}\right)^2 &= \norm{\tildeF B_i - \wvec_i}_2^2 \le \frac{\wvec_i^\top \laplacian^2 \wvec_i}{\lambda_{k+1}^2} \le \frac{2\alpha^2}{\lambda_{k+1}^2},
	\end{align}
	where the first inequlity uses the fact that $\tildeF B_i$ is the projection of $\wvec_i$ onto the top $k$ eigenspan, and the second inequality is by Claim~\ref{lemma:laplaciansquare}. 
	Let $\tau:\data\rightarrow [\ncluster]$ be the cluster index function such that $x\in C_{\tau(x)}$ for $x\in\data$.
	Summing the above equation over $i\in[\ncluster]$ gives
	\begin{align}
		\Exp{x\sim\pdata}{\norm{B^\top f(x)-\idvec{\clusterf{x}}}_2^2} \le \frac{2\ncluster\alpha^2}{\lambda_{k+1}^2}.
	\end{align}
	Finally, we finish the proof by noticing that $g_{f, B}(x)\ne\tau(x)$ only if $\norm{B^\top f(x) - \idvec{\clusterf{x}}}_2^2 \ge \frac{1}{2}$.
	
\end{proof}

\section{Proof of Theorem~\ref{theorem:uda}}\label{appendix:uda}
We prove the following theorem which directly implies Theorem~\ref{theorem:uda}.
\begin{theorem}\label{theorem:uda_appendix}
	Suppose that Assumption
	\ref{assumption:separation} and \ref{assumption:uda} holds, and $\pdata(\source)/\pdata(\target) \le O(1)$.
	Let $f$ be a minimizer of the contrastive loss $\loss{2}(\cdot)$ and the head $g_t$ be defined in~\eqref{eqn:13}.
	Then, for any  $1\le t \le \rho/(8\alpha^2)$, we have
	\begin{align}
			\Err_\target(g_t) \lesssim \frac{\nclass}{\alpha^2\lambda_{k+1}^2} \cdot \big(1-\lambda_{k+1}/2\big)^{t},\label{eqn:15}
		\end{align}
where $\lambda_{k+1}$ is the $k$+$1$-th smallest eigenvalue of the Laplacian of the positive-pair graph.
\end{theorem}

We first introduce the following lemma, which says that the indicator vector of a cluster wouldn't change much after multiplying $\norA$ a few times.
\begin{lemma}\label{lemma:induction}
	Suppose Assumption~\ref{assumption:separation} holds.
	For every $i\in[\ncluster]$, define $\wvec_i\in\Real^\Ndata$ be such that the $x$-th dimension of it is \begin{align}
		(\wvec_i)_x = \begin{cases}
			\sqrt{\ww{x}} & \text{if } x\in \clusterset_i\\
			0 & \text{otherwise}
		\end{cases} 
	\end{align}
	Then, for any two clusters $i\ne j$ in $[\ncluster]$, the following holds for any integer $t\in [0,\frac{1}{\alpha}]$:
	\begin{itemize}
		\item{For any $x\in\clusterset_i$,  we have 
			\begin{align}\label{equation:induction_eq2}
				\left(\left(\frac{1}{2}\imatrix+\frac{1}{2}\norA\right)^t \wvec_i\right)_x \in \left[(1-t\alpha) \sqrt{\ww{x}}, \sqrt{\ww{x}}\right].
		\end{align}}
		\item{For any $x\notin \clusterset_i$, we have 
			\begin{align}\label{equation:induction_eq3}
				\left(\left(\frac{1}{2}\imatrix+\frac{1}{2}\norA\right)^t \wvec_i\right)_x \in \left[0, t\alpha\sqrt{\ww{x}}\right].
		\end{align}}
	\end{itemize}
\end{lemma}

\begin{proof}[Proof of Lemma~\ref{lemma:induction}]
	We prove this lemma by induction. When $t=0$, obviously equations \eqref{equation:induction_eq2} and \eqref{equation:induction_eq3} are all true. Assume they are true for $t=l$, we prove that they are still true at $t=l+1$ so long as $l\le \frac{1}{\alpha}$. We define shorthands $\wvec'_i = \left(\frac{1}{2}\imatrix + \frac{1}{2}\norA\right)^l \wvec_i$ and $\wvec'_j = \left(\frac{1}{2}\imatrix + \frac{1}{2}\norA\right)^l \wvec_j$.
	
	For the induction of Equation~\eqref{equation:induction_eq2}, let $x\in\clusterset_i$. On one hand, we have
	\begin{align}
		\sqrt{\ww{x}}\left(\norA \wvec'_i\right)_x &= \sum_{\tilde{x}\in \clusterset_i} \frac{\ww{x, \tilde{x}}}{\sqrt{\ww{\tilde{x}}}} (\wvec'_i)_{\tilde{x}}+ \sum_{\tilde{x}\notin\clusterset_i}\frac{\ww{x, \tilde{x}}}{\sqrt{\ww{\tilde{x}}}}(\wvec'_i)_{\tilde{x}}\\
		&\le \sum_{\tilde{x}\in \clusterset_i} \frac{\ww{x, \tilde{x}}}{\sqrt{\ww{\tilde{x}}}} \sqrt{\ww{\tilde{x}}}+ \sum_{\tilde{x}\notin\clusterset_i}\ww{x, \tilde{x}}(l\alpha) \\
		&\le \sum_{\tilde{x}\in\data} \ww{x, \tilde{x}} = \ww{x},
	\end{align}
	where the first inequality uses Equations~\eqref{equation:induction_eq2} and \eqref{equation:induction_eq3} at $t=l$, and the second inquality uses $l \le \frac{1}{\alpha}$.
	On the other hand, we have 
	\begin{align}
		\sqrt{\ww{x}}\left(\norA \wvec'_i\right)_x &= \sum_{\tilde{x}\in \clusterset_i} \frac{\ww{x, \tilde{x}}}{\sqrt{\ww{\tilde{x}}}} (\wvec'_i)_{\tilde{x}}+ \sum_{\tilde{x}\notin\clusterset_i}\frac{\ww{x, \tilde{x}}}{\sqrt{\ww{\tilde{x}}}}(\wvec'_i)_{\tilde{x}}\\
		&\ge \sum_{\tilde{x}\in\clusterset_i} \frac{\ww{x, \tilde{x}}}{\sqrt{\ww{\tilde{x}}}} (1-l\alpha) \sqrt{\ww{\tilde{x}}} \\
		&\ge (1-l\alpha)(1-\alpha) \ww{x} \ge (1-(l+1)\alpha)\ww{x},
	\end{align}
	where the first inequality uses Equations~\eqref{equation:induction_eq2} and \eqref{equation:induction_eq3} at $t=l$, and the second inquality uses the definition of $\alpha$-max-connection. Combining them gives us $\sqrt{\ww{x}}\left(\norA \wvec'_i\right)_x \in [(1-(l+1))\sqrt{\ww{x}}, \sqrt{\ww{x}}]$, which directly leads to 
	\begin{align}
		\left(\left(\frac{1}{2}\imatrix+\frac{1}{2}\norA\right)^{l+1} \wvec_i\right)_x = \frac{1}{2}(\wvec'_i)_x + \frac{1}{2} (\norA \wvec'_i)_x \in \left[(1-(l+1)\alpha)\sqrt{\ww{x}}, \sqrt{\ww{x}}\right].
	\end{align}
	
	For the induction of Equation~\eqref{equation:induction_eq3}, let $x\notin\clusterset_i$. Since $\norA$ and $\wvec_i$ are both element-wise nonnegative, we have $\norA \wvec'_i$ is element-wise nonnegative, hence $(\norA \wvec'_i)_x\ge 0$. On the other hand, we have 
	\begin{align}
		\sqrt{\ww{x}}\left(\norA \wvec'_i\right)_x &= \sum_{\tilde{x}\in\clusterset_i} \frac{\ww{x, \tilde{x}}}{\sqrt{\ww{\tilde{x}}}} (\wvec'_i)_{\tilde{x}} + \sum_{\tilde{x}\notin\clusterset_i} \frac{\ww{x, \tilde{x}}}{\sqrt{\ww{\tilde{x}}}} (\wvec'_i)_{\tilde{x}} \\
		&\le  \sum_{\tilde{x}\in\clusterset_i} \frac{\ww{x, \tilde{x}}}{\sqrt{\ww{\tilde{x}}}} \sqrt{\ww{\tilde{x}}} +  l\alpha \cdot \sum_{\tilde{x}\notin\clusterset_i} \frac{\ww{x, \tilde{x}}}{\sqrt{\ww{\tilde{x}}}} \sqrt{\ww{\tilde{x}}} \\
		&\le \alpha \ww{x} + l\alpha \ww{x} = (l+1)\alpha \ww{x},
	\end{align}
	where the first inequality uses Equations~\eqref{equation:induction_eq2} and \eqref{equation:induction_eq3} at $t=l$, and the second inequality is by $\alpha$-max-connection.
	Hence we have $(\norA \wvec'_i)_x \in [0, (l+1)\alpha \ww{x}]$ which directly leads to 
	\begin{align}
		\left(\left(\frac{1}{2}\imatrix+\frac{1}{2}\norA\right)^{l+1} \wvec_i\right)_x = \frac{1}{2}(\wvec'_i)_x + \frac{1}{2} (\norA \wvec'_i)_x \in \left[0, (l+1)\alpha\sqrt{\ww{x}}\right].
	\end{align}
	
\end{proof}

The following lemma shows that a random walk starting from $\target_i$ is more likely to arrive at $\source_i$ than in $\source_j$ for $j\ne i$.
\begin{lemma}\label{lemma:induction_target}
	Suppose that Assumptions~\ref{assumption:separation} and ~\ref{assumption:uda} hold.
	For every $i\in[\nclass]$, define $\wvec_i\in\Real^\Ndata$ be such that the $x$-th dimension of it is \begin{align}
		(\wvec_i)_x = \begin{cases}
			\sqrt{\ww{x}} & \text{if } x\in \source_i\\
			0 & \text{otherwise}
		\end{cases} 
	\end{align}
	Then, for any two classes $i\ne j$ in $[\nclass]$, we have the following holds for any integer $t\in [0,\frac{\rho}{8\alpha^2}]$ and $x\in\target_i$:
	\begin{align}\label{equation:induction_eq1}
		\left(\left(\frac{1}{2}\imatrix+\frac{1}{2}\norA\right)^t \wvec_i\right)_x - \left(\left(\frac{1}{2}\imatrix+\frac{1}{2}\norA\right)^t \wvec_j\right)_x \ge \begin{cases}
			0 & \text{if } t=0\\
			\frac{1}{4}\rho\sqrt{\ww{x}} & \text{if } t\ge 1
		\end{cases}.
	\end{align}
\end{lemma}

\begin{proof}[Proof of Lemma~\ref{lemma:induction_target}]
	We prove this lemma by induction. When $t=0$, obviously equation \eqref{equation:induction_eq1} is true. Assume it is true for $t=l$, we prove that they are still true at $t=l+1$ so long as $l\le \frac{\rho}{8\alpha^2}$.
	
	We define shorthands $\wvec'_i = \left(\frac{1}{2}\imatrix + \frac{1}{2}\norA\right)^l \wvec_i$ and $\wvec'_j = \left(\frac{1}{2}\imatrix + \frac{1}{2}\norA\right)^l \wvec_j$.
	
	Let $x\in\target_i$, we notice that 
	\begin{align}
		\sqrt{\ww{x}} \left(\norA \wvec'_i - \norA \wvec'_j\right)_x =\underbrace{\sum_{\tilde{x}\in\source_i} \frac{\ww{x, \tilde{x}}}{\sqrt{x_{\tilde{x}}}} \left((\wvec'_i)_{\tilde{x}} - (\wvec'_j)_{\tilde{x}}\right)}_{Q_1} + \underbrace{\sum_{\tilde{x}\in\source_j} \frac{\ww{x, \tilde{x}}}{\sqrt{x_{\tilde{x}}}} \left((\wvec'_i)_{\tilde{x}} - (\wvec'_j)_{\tilde{x}}\right)}_{Q_2} \\ + \underbrace{\sum_{\tilde{x}\in\target_i} \frac{\ww{x, \tilde{x}}}{\sqrt{x_{\tilde{x}}}} \left((\wvec'_i)_{\tilde{x}} - (\wvec'_j)_{\tilde{x}}\right)}_{Q_3} + \underbrace{\sum_{\tilde{x}\notin\source_i\cup\source_j\cup\target_i} \frac{\ww{x, \tilde{x}}}{\sqrt{x_{\tilde{x}}}} \left((\wvec'_i)_{\tilde{x}} - (\wvec'_j)_{\tilde{x}}\right)}_{Q_4}
	\end{align}
	
	Since $\rho\le\alpha$ must be true for the assumptions to be valid, we know $l\le\frac{\rho}{8\alpha^2}\le\frac{1}{\alpha}$, hence we apply Lemma~\ref{lemma:induction} and have Equations \eqref{equation:induction_eq2} and \eqref{equation:induction_eq3} hold at $t=l$.
	Using them together with Equation \eqref{equation:induction_eq1} at $t=l$ and Assumption~\ref{assumption:uda}, we have
	\begin{align}
		Q_1 \ge \sum_{\tilde{x}\in\source_i} \frac{\ww{x, \tilde{x}}}{\sqrt{\ww{\tilde{x}}}} (1-2l\alpha)\sqrt{\ww{\tilde{x}}} \ge (1-2l\alpha)\rho \ww{x},
	\end{align} 
	\begin{align}
		Q_2 \ge -\sum_{\tilde{x}\in\source_j} \frac{\ww{x, \tilde{x}}}{\sqrt{\ww{\tilde{x}}}} \sqrt{\ww{\tilde{x}}}\ge -\frac{\rho}{c} \ww{x},
	\end{align} 
	\begin{align}
		Q_3\ge 0,
	\end{align}
	and 
	\begin{align}
		Q_4 \ge -\sum_{\tilde{x}\notin\source_i\cup\source_j\cup\target_i} \frac{\ww{x, \tilde{x}}}{\sqrt{\ww{\tilde{x}}}} (l\alpha) \sqrt{\ww{x}} \ge -l\alpha^2 \ww{x}.
	\end{align}
	Combining them gives us 
	\begin{align}
		\sqrt{\ww{x}} \left(\norA \wvec'_i - \norA \wvec'_j\right)_x \ge \left(\rho - \left(\frac{\rho}{c} + 2l\alpha\rho +l\alpha^2\right)\right)\ww{x}.
	\end{align}
	Since $\frac{\rho}{c}\le\frac{1}{8}\rho$, $l\le \frac{\rho}{8\alpha^2}$ and $\rho\le \alpha$, we have $\sqrt{\ww{x}} \left(\norA \wvec'_i - \norA \wvec'_j\right)_x \ge \frac{1}{2}\rho \ww{x}$ hence $(\norA \wvec'_i - \norA \wvec'_j)_x \ge \frac{1}{2}\rho\sqrt{\ww{x}}$. As a result, we have 
	\begin{align}
		\left(\left(\frac{1}{2}\imatrix+\frac{1}{2}\norA\right)^{l+1} \wvec_i\right)_x - \left(\left(\frac{1}{2}\imatrix+\frac{1}{2}\norA\right)^{l+1} \wvec_j\right)_x \ge \frac{1}{4}\rho\sqrt{\ww{x}}.
	\end{align}
	
\end{proof}

The following lemma shows that the power of $\norA$ can be low-rank approximated with a small error.
\begin{lemma}\label{lemma:power_t_error}
	Suppose that Assumption~\ref{assumption:separation} holds.
	For every $i\in[\nclass]$, define $\wvec_i\in\Real^\Ndata$ be such that the $x$-th dimension of it is \begin{align}
		(\wvec_i)_x = \begin{cases}
			\sqrt{\ww{x}} & \text{if } x\in \source_i\\
			0 & \text{otherwise}
		\end{cases} 
	\end{align}
	Let $f:\data\rightarrow\Real^k$ be a minimizer of the generalized spectral contrastive loss $\loss{2}(\cdot)$. 
	Define matrix $\tildeF\in\Real^{\Ndata\times k}$ be such that the $x$-th row of it contains $\sqrt{\ww{x}}\cdot f(x)$. Then, we have
	\begin{align}
		\norm{\left(\frac{1}{2}\imatrix+\frac{1}{2}\norA\right)^t \wvec_i - \left(\tildeF\tildeF^\top\right)^t \wvec_i}_2^2 \le \frac{2\epsilon_t\alpha^2}{\lambda_{k+1}^2} \norm{\wvec_i}_2^2,
	\end{align}
	where 
	\begin{align}
		\epsilon_t = (1-\frac{1}{2}\lambda_{k+1})^{2t}.
	\end{align}
	\begin{proof}[Proof of Lemma~\ref{lemma:power_t_error}]
		Let $\Pi_k(\wvec_i)$ be the projection of $\wvec_i$ onto the column span of $\tildeF$.	
		Notice that every eigenvalue of $\laplacian$ is in the range $[0, 2]$, by Theorem~\ref{theorem:alpha2bound} we have 
		\begin{align}
			\norm{\wvec_i - \Pi_k(\wvec_i)}_2^2 \le \frac{2\alpha^2}{\lambda_{k+1}^2}. 
		\end{align}
		Therefore, notice that $\tildeF\tildeF^\top$ is exactly the top $k$ components of $\frac{1}{2}\imatrix+\frac{1}{2}\norA$, we have
		\begin{align}
			\norm{\left(\frac{1}{2} \imatrix+\frac{1}{2}\norA\right)^t \wvec_i - \left(\tildeF\tildeF^\top\right)^t \wvec_i}_2^2 &\le \left(1-\frac{1}{2}\lambda_{k+1}\right)^{2t} \norm{\wvec_i - \Pi_k(\wvec_i)}_2^2 &\le \frac{2\epsilon_t\alpha^2}{\lambda_{k+1}^2} \norm{\wvec_i}_2^2.
		\end{align}
	\end{proof}

\end{lemma}

Using the above lemmas, we finish the proof of Theorem~\ref{theorem:uda_appendix}.
\begin{proof}[Proof of Theorem~\ref{theorem:uda_appendix}]
	For every $i\in[\nclass]$, define $\wvec_i\in\Real^\Ndata$ be such that the $x$-th dimension of it is \begin{align}
		(\wvec_i)_x = \begin{cases}
			\sqrt{\ww{x}} & \text{if } x\in \source_i\\
			0 & \text{otherwise}
		\end{cases} 
	\end{align}
	Define matrix $\tildeF\in\Real^{\Ndata\times k}$ be such that the $x$-th row of it contains $\sqrt{\ww{x}}\cdot f(x)$. 
	
	Let $i\ne j$ be two different classes in $[\nclass]$. 
	By Lemma~\ref{lemma:power_t_error} we know that 
	\begin{align}\label{equation:predict_bound_on_c}
		\norm{\left(\frac{1}{2}\imatrix+\frac{1}{2}\norA\right)^t \wvec_i - \left(\tildeF\tildeF^\top\right)^t \wvec_i}_2^2 \le \frac{2\epsilon_t\alpha^2}{\lambda_{k+1}^2} \norm{\wvec_i}_2^2,
	\end{align}
	and 
	\begin{align}\label{equation:predict_bound_on_c_prime}
		\norm{\left(\frac{1}{2}\imatrix+\frac{1}{2}\norA\right)^t \wvec_j - \left(\tildeF\tildeF^\top\right)^t \wvec_j}_2^2 \le \frac{2\epsilon_t\alpha^2}{\lambda_{k+1}^2} \norm{\wvec_j}_2^2.
	\end{align}
	Define shorthand 
	\begin{align}
		Q_{i,j} = \left(\left(\tildeF\tildeF^\top\right)^t \wvec_i - \left(\tildeF\tildeF^\top\right)^t \wvec_j \right) - \left(\left(\frac{1}{2}\imatrix+\frac{1}{2}\norA\right)^t \wvec_i - \left(\frac{1}{2}\imatrix+\frac{1}{2}\norA\right)^t \wvec_j\right).
	\end{align}
	From Equations~\eqref{equation:predict_bound_on_c} and \eqref{equation:predict_bound_on_c_prime} we have 
	\begin{align}\label{equation:diff_bound}
		\norm{Q_{i, j}}_2^2 \le \frac{4\epsilon_t\alpha^2}{\lambda_{k+1}^2} \left(\norm{\wvec_i}_2^2 + \norm{\wvec_j}_2^2\right).
	\end{align}
	Recall that 
	\begin{align}
		\precond = \Exp{x\sim\pdata}{f(x)f(x)^\top} = \tildeF^\top \tildeF,
	\end{align}
	and for $i\in[\nclass]$,
	\begin{align}
		\head_i = \Exp{x\sim \psource}{\id{x\in\source_i} \cdot f(x)} =\frac{ \tildeF^\top \wvec_i}{\Prob(\source)}.
	\end{align}
	We can rewrite the prediction for any $x\in\target$,
	\begin{align}
		\pred_t(x) = \argmax_{i\in[\nclass]} f(x)^\top \precond^{t-1}\head_i =  \argmax_{i\in[\nclass]} \left((\tildeF\tildeF^\top)^t \wvec_i\right)_x.
	\end{align}
	
	Therefore, for $x\in\target_i$, in order for $\pred_t(x)=j\ne i$, there must be 
	\begin{align}
		\left(\left(\tildeF\tildeF^\top\right)^t \wvec_i - \left(\tildeF\tildeF^\top\right)^t \wvec_j \right)_x \le 0.
	\end{align}
	On the other hand, we know from Lemma~\ref{lemma:induction_target} that 
	\begin{align}
		\left(\left(\frac{1}{2}\imatrix+\frac{1}{2}\norA\right)^t \wvec_i - \left(\frac{1}{2}\imatrix+\frac{1}{2}\norA\right)^t \wvec_j\right)_x \ge \frac{1}{4}\rho\sqrt{\ww{x}}.
	\end{align}
	Therefore, whenever  $x\in\target_i$, in order for $\pred_t(x)=j$, there has to be 
	\begin{align}\label{equation:min_diff}
		(Q_{ij})_x \le -\frac{1}{4}\rho\sqrt{\ww{x}}.
	\end{align}
	
	Finally, we can bound the target error as follows:
	\begin{align}
		\Exp{x\sim\ptarget}{\id{\pred_t(x)\ne \yf{x}}} &= \frac{1}{\Prob(\target)} \sum_{x\in\target} \id{\pred_t(x)\ne \yf{x}} \cdot \ww{x}\\
		&= \frac{1}{\Prob(\target)} \sum_{i\in[r]} \sum_{j\ne i}\sum_{x\in\target_i} \id{\pred_t(x)=j}\cdot \ww{x}\\
		&\le \frac{1}{\Prob(\target)} \sum_{i\in[r]} \sum_{j\ne i}\sum_{x\in\target_i} \frac{(Q_{ij})_x^2 \ww{x}}{\frac{1}{16}\rho^2 \ww{x}}\\
		&\le \frac{1}{\Prob(\target)} \frac{32\nclass}{\rho^2} \cdot \frac{4\epsilon_t\alpha^2}{\lambda_{k+1}^2} \sum_{i\in[\nclass]} \norm{\wvec_i}_2^2\\
		&= \frac{128\epsilon_t\nclass\alpha^2}{\rho^2\lambda_{k+1}^2} \cdot \frac{\Prob(\source)}{\Prob(\target)},
	\end{align}
	where the first inequality is from Equation~\eqref{equation:min_diff} and the second inequality follows Equation~\eqref{equation:diff_bound}.
	Notice that Assumption~\ref{assumption:uda} we have $\alpha^2 \lesssim \rho$, hence we finish the proof.
	
\end{proof}

\section{Proof of Theorem~\ref{theorem:uda_multistep}}\label{appendix:uda_multistep}

We prove the following theorem which directly implies Theorem~\ref{theorem:uda_multistep}.
\begin{theorem}\label{theorem:uda_multistep_appendix}
	Suppose Assumptions \ref{assumption:separation}, \ref{assumption:uda_multistep_interclass} and \ref{assumption:intra_class_conductance} hold and $\pdata(\source)/\pdata(\target) \le O(1)$.
	Let $g_t$ be defined the same as in Theorem~\ref{theorem:uda}.
	Then, for any $1\le t\le\frac{1}{\alpha}$, we have
\begin{align}
		\Err_\target(g_t) \lesssim\frac{\nclass}{\lambda_{k+1}^2}\cdot \max\big\{\frac{1}{\tau^2\alpha^4}\big(1-\frac{1}{4}\min\{\gamma^2, \lambda_{k+1}\}\big)^{t}, \frac{t^2}{\tau}\big\},
\end{align}
where $\lambda_{k+1}$ is the $k$+$1$-th smallest eigenvalue of the Laplacian of the positive-pair graph.
\end{theorem}

For every $i\in[\nclass]$, we consider a graph $\graph(\target_i, w)$ that is $\graph(\data, w)$ restricted on $\target_i$. 
We use $\lambda_{\target_i}$ to denote the second smallest eigenvalue of the Laplacian of $\graph(\target_i, w)$. 
For $x\in\target_i$, we use $\hat{w}_x = \sum_{x'\in\target_i} \ww{x, x'}$ to denote the total weight of $x$ in the restricted graph $\graph(\target_i, w)$.
We use $\norA_{\target_i}$ to denote the normalized adjacency matrix of $\graph(\target_i, w)$.

The following lemma shows the relationship between intra-class expansion and the eigvenvalue of the restricted graph's Laplacian.
\begin{lemma}\label{lemma:cheeger}
	Suppose that Assumption~\ref{assumption:intra_class_conductance} holds. Then, we have
	\begin{align}
	\lambda_{\target_i}\ge \frac{\gamma^2}{2}.
	\end{align}
\end{lemma}
\begin{proof}
For set $H\subset \target_i$, we use $\hat{w}(H) = \sum_{x\in H, x'\in\target_i} \ww{x, x'}$ to denote the size of set $S$ in restricted graph  $\graph(\target_i, w)$. Clearly $\hat{w}(H)\le w(H)$. We have
\begin{align}
	\min_{H\subseteq \target_i} \frac{w(H, \target_i\backslash H)}{\min\{\hat{w}(H), \hat{w}(\target_i\backslash H)\}} \ge \min_{H\subseteq \target_i} \frac{w(H, \target_i\backslash H)}{\min \{{w}(H), {w}(\target_i\backslash H)\}} \ge \gamma.
\end{align}
Directly applying Cheeger's Inequality finishes the proof.
\end{proof}

For every $i\in[\nclass]$, define $\wvec_i\in\Real^\Ndata$ be such that the $x$-th dimension of it is \begin{align}
	(\wvec_i)_x = \begin{cases}
		\sqrt{\ww{x}} & \text{if } x\in \source_i\\
		0 & \text{otherwise}
	\end{cases} 
\end{align}

The following lemma lower bounds the probability that a random walk starting from $\target_i$ arrives at $\source_i$.
\begin{lemma}\label{lemma:uda_multistep_rho}
	Suppose that Assumption~\ref{assumption:separation} holds.
	For every $i\in[\nclass]$ and $t\ge0$, there exists vectors $\Delta_i\in\Real^{|\target_i|}$ such that for any $x\in\target_i$,
	\begin{align}
		\left(\left(\frac{1}{2}\imatrix + \frac{1}{2}\norA\right)^t \wvec_i\right)_x \ge \frac{1}{2}(1-\alpha)^t \rho_i \sqrt{\ww{x}} + (\Delta_i)_x,
	\end{align}
	where $\rho_i := \phi(\target_i, \source_i)$, and 
	\begin{align}
		\norm{\Delta_i}^2 \le \left(1-\frac{\lambda_{\target_i}}{2}\right)^{2(t-1)} \Prob(\target_i).
	\end{align}
\end{lemma}
\begin{proof}[Proof of Lemma~\ref{lemma:uda_multistep_rho}]
Recall that $\norA_{\target_i}$ is the normalized adjacency matrix of the restircted graph on $\target_i$. 
We first notice that for any $x, x'\in\target_i$, 
\begin{align}
\left(\frac{1}{2}\imatrix + \frac{1}{2}\norA\right)_{xx'} \ge (1-\alpha) \left(\frac{1}{2}\imatrix + \frac{1}{2}\norA_{\target_i}\right)_{xx'},
\end{align}
where we use the Assumption~\ref{assumption:separation}. Thus, we have
\begin{align}
	\left(\left(\frac{1}{2}\imatrix + \frac{1}{2}\norA\right)^t \wvec_i \right)_{\target_i} \ge \frac{1}{2}(1-\alpha)^{t-1} \left(\frac{1}{2}\imatrix + \frac{1}{2}\norA_{\target_i}\right)^{t-1} \left(\norA \wvec_i\right)_{\target_i},
\end{align}
here we use $(\cdot)_{\target_i}$ to denote restricting a vector in $\Real^{\Ndata}$ to those dimensions in $\target_i$. 

Let vector $u\in\Real^{|\target_i|}$ be such that its $x$-th dimension is $\sqrt{{w}_x}$, $\tilde{u}\in\Real^{|\target_i|}$ be such that its $x$-th dimension is $\sqrt{\hat{w}_x}$. It's standard result that $u$ is the top eigenvector of $\norA_{\target_i}$ with eigenvalue 1. Let $v_1$ be the projection of vector $\left(\norA \wvec_i\right)_{\target_i}$ onto $\tilde{u}$ and $v_2 = \left(\norA \wvec_i\right)_{\target_i}-v_1$. We have
\begin{align}
	\left(\frac{1}{2}\imatrix + \frac{1}{2}\norA_{\target_i}\right)^{t-1}  v_1 = v_1
	&= \frac{\tilde{u}^\top \left(\norA \wvec_i\right)_{\target_i}}{\norm{\tilde{u}}^2} \tilde{u} 
	&\ge (1-\alpha) \frac{{u}^\top \left(\norA \wvec_i\right)_{\target_i}}{\norm{{u}}^2} {u} \ge (1-\alpha)\rho_i u.
\end{align}
\begin{align}
	\norm{\left(\frac{1}{2}\imatrix + \frac{1}{2}\norA_{\target_i}\right)^{t-1}  v_2} \le \left(1-\frac{\lambda_{\target_i}}{2}\right)^{t-1}\norm{v_2}\le \left(1-\frac{\lambda_{\target_i}}{2}\right)^{t-1}\norm{\left(\norA \wvec_i\right)_{\target_i}}\\
	\le \left(1-\frac{\lambda_{\target_i}}{2}\right)^{t-1}\norm{u} \le  \left(1-\frac{\lambda_{\target_i}}{2}\right)^{t-1}\sqrt{\Prob(\target_i)}.
\end{align}
Setting $\Delta_i = \frac{1}{2}(1-\alpha)^{t-1} \left(\frac{1}{2}\imatrix + \frac{1}{2}\norA_{\target_i}\right)^{t-1}  v_2$ finishes the proof.
\end{proof}

The following lemma upper bounds the probability that a random walk starting from $\target_i$ arrives at $\source_j$ for $j\ne i$.
\begin{lemma}\label{lemma:uda_multistep_beta}
	Suppose that Assumption~\ref{assumption:separation} holds.
	For every $i\ne j$ in $[\nclass]$ and $t\in[0, \frac{1}{\alpha}]$, we have
	\begin{align}\label{equation:induction_uda_beta}
		\sum_{x\in\target_i} \sqrt{\ww{x}} \left(\left(\frac{1}{2}\imatrix + \frac{1}{2}\norA\right)^t \wvec_j\right)_x \le (t^2\alpha^2+t\beta_{i,j}) \Prob(\target_i),
	\end{align}
where $\beta_{i,j} := \phi(\target_i, \source_j)$ .
\end{lemma}
\begin{proof}[Proof of Lemma~\ref{lemma:uda_multistep_beta}]
	We prove with induction. When $t=0$ clearly Equation~\ref{equation:induction_uda_beta} is true. Assume Equation~\ref{equation:induction_uda_beta} holds for $t=l$. Define shorthand
	\begin{align}
		\wvec'_j = \left(\frac{1}{2}\imatrix + \frac{1}{2}\norA\right)^l \wvec_j.
	\end{align}
We have
\begin{align}
	\sum_{x\in\target_i} \sqrt{\ww{x}}\left(\left(\frac{1}{2}\imatrix + \frac{1}{2}\norA\right)^{l+1} \wvec_j\right)_x = \frac{1}{2}\sum_{x\in\target_i} \sqrt{\ww{x}} (\wvec'_j)_x + \frac{1}{2}\underbrace{\sum_{x\in\target_i}\sum_{x'\in\target_i} \sqrt{\ww{x}}\norA_{xx'}(\wvec'_j)_{x'}}_{Q_1}  \\
	+ \frac{1}{2}\underbrace{\sum_{x\in\target_i}\sum_{x'\in\source_j} \sqrt{\ww{x}}\norA_{xx'}(\wvec'_j)_{x'}}_{Q_2} + \frac{1}{2}\underbrace{\sum_{x\in\target_i}\sum_{x'\notin \target_i\cup\source_j} \sqrt{\ww{x}}\norA_{xx'}(\wvec'_j)_{x'}}_{Q_3}
\end{align}
Using Equation~\ref{equation:induction_uda_beta} at $t=l$, we have
\begin{align}
	Q_1\le \sum_{x'\in\target_i}\sqrt{\ww{x'}} (\wvec'_j)_{x'} \le (l^2\alpha^2+l\beta_{i,j}) \Prob(\target_i).
\end{align}
Lemma~\ref{lemma:induction} tells us $(\wvec_i)_{x'}\le \sqrt{\ww{x'}}$ for $x'\in\source_j$, so by the definition of $\beta_{i,j}$ we have
\begin{align}
	Q_2 \le \sum_{x\in\target_i}\sum_{x'\in\source_j} \sqrt{\ww{x}}\norA_{xx'}\sqrt{\ww{x'}} \le \beta_{i,j}\Prob(\target_i).
\end{align}
Lemma~\ref{lemma:induction} also tells us $(\wvec_i)_{x'}\le l\alpha\sqrt{\ww{x'}}$ for $x'\notin\source_j$, so by Assumption~\ref{assumption:separation} we have
\begin{align}
	Q_3 \le l\alpha \sum_{x\in\target_i}\sum_{x'\notin\target_i\cup\source_j} \sqrt{\ww{x}}\norA_{xx'}\sqrt{\ww{x'}} \le l\alpha^2\Prob(\target_i).
\end{align}
Adding these three terms finishes the proof for $t=l+1$.
\end{proof}

Now we use the above lemmas to finish the proof of Theorem~\ref{theorem:uda_multistep_appendix}.
\begin{proof}[Proof of Theorem~\ref{theorem:uda_multistep_appendix}]
For $i\ne j\in[\nclass]$, define
	\begin{align}
		Q_{i,j} := \left(\left(\tildeF\tildeF^\top\right)^t \wvec_i - \left(\tildeF\tildeF^\top\right)^t \wvec_j \right)_{\target_i} - \left(\left(\frac{1}{2}\imatrix+\frac{1}{2}\norA\right)^t \wvec_i - \left(\frac{1}{2}\imatrix+\frac{1}{2}\norA\right)^t \wvec_j\right)_{\target_i}.
	\end{align}
Let $\Delta_i$ be the vector in Lemma~\ref{lemma:uda_multistep_rho}, and 
\begin{align}
\Lambda_j := \left(\left(\frac{1}{2}\imatrix + \frac{1}{2}\norA\right)^t \wvec_j\right)_{\target_i}.
\end{align}
Using Lemma~\ref{lemma:uda_multistep_rho} and $t\le\frac{1}{2\alpha}$, we know for $x\in\target_i$, 
\begin{align}
	\left(\left(\tildeF\tildeF^\top\right)^t \wvec_i - \left(\tildeF\tildeF^\top\right)^t \wvec_j \right)_{x} &\ge \frac{1}{2}(1-\alpha)^t \rho \sqrt{\ww{x}} + \left(Q_{i, j}\right)_x  + (\Delta_i)_x - \left(\Lambda_j\right)_x\\
	&\ge \frac{1}{4}\rho_i\sqrt{\ww{x}} + \left(Q_{i, j}\right)_x  + (\Delta_i)_x - \left(\Lambda_j\right)_x,
\end{align}
where $\rho_i=\phi(\target_i, \source_i)$.

When $\pred_t(x) = j$, at least one of $|\left(\Delta_i\right)_x|$, $|(Q_{i, j})_x|$ and $(\Lambda_j)_x$ is at least $\frac{1}{12}\rho_i\sqrt{\ww{x}}$.
Thus, we have
\begin{align}
	\sum_{x\in\target_i} \ww{x} \id{\pred_{t}(x)=j} &
	\le \sum_{x\in\target_i}\ww{x} \id{(\Delta_i)_x^2 \ge \frac{1}{144}\rho_i^2 \ww{x}} + \sum_{x\in\target_i}\ww{x} \id{(Q_{i,j})_x^2 \ge \frac{1}{144}\rho_i^2 \ww{x}} \\
	&+ \sum_{x\in\target_i}\ww{x} \id{(\Lambda_j)_x \ge \frac{1}{12}\rho_i \sqrt{\ww{x}}}\\
	&\le\frac{144}{\rho_i^2} \norm{\Delta_i}_2^2 + \frac{144}{\rho_i^2} \norm{Q_{i,j}}_2^2 + \frac{12}{\rho_i}\sum_{x\in\target_i} \sqrt{\ww{x}} \left(\left(\frac{1}{2}\imatrix + \frac{1}{2}\norA\right)^t \wvec_j\right)_x \label{equation:uda_mutlstep_eq4}
\end{align}
Using Lemma~\ref{lemma:power_t_error} we know
\begin{align}\label{equation:uda_mutlstep_eq1}
	\norm{Q_{i,j}}_2^2 \le \frac{4\epsilon_t\alpha^2}{\lambda_{k+1}^2} \left(\Prob(\source_i) + \Prob(\source_j)\right),
\end{align}
where 
\begin{align}
	\epsilon_t := (1-\frac{1}{2}\lambda_{k+1})^{2t}.
\end{align}
Using Lemma~\ref{lemma:uda_multistep_rho} and Lemma~\ref{lemma:cheeger} we know
\begin{align}\label{equation:uda_mutlstep_eq2}
	\norm{\Delta_i}_2^2 \le \left(1-\frac{\gamma^2}{4}\right)^{2(t-1)} \Prob(\target_i).
\end{align}
Using Lemma~\ref{lemma:uda_multistep_beta} we know
\begin{align}\label{equation:uda_mutlstep_eq3}
\sum_{x\in\target_i} \sqrt{\ww{x}} \left(\left(\frac{1}{2}\imatrix + \frac{1}{2}\norA\right)^t \wvec_j\right)_x \le (t^2\alpha^2+t\beta_{i,j}) \Prob(\target_i),
\end{align}
where $\beta_{i,j} := \phi(\target_i, \source_j)$ .

Let $\rho:=\min_{i\in[\nclass]} \rho_i$.
Plugging Equations \eqref{equation:uda_mutlstep_eq1}, \eqref{equation:uda_mutlstep_eq2} and \eqref{equation:uda_mutlstep_eq3} into Equation~\eqref{equation:uda_mutlstep_eq4} and summing over all $i$ and $j$ gives
\begin{align}
	\sum_{x\in\target} \ww{x}\id{\pred_{t}(x)\ne \yf{x}} \le \frac{144\nclass}{\rho^2}\left(1-\frac{\gamma^2}{4}\right)^{2(t-1)} \Prob(\target) + \frac{1152\nclass\epsilon_t\alpha^2}{\rho^2\lambda_{k+1}^2} \Prob(\source) \\
	+ \frac{12\nclass t^2\alpha^2}{\rho}\Prob(\target) + \max_{i\ne j}\left\{\frac{\beta_{i, j}}{\rho_i}\right\} 12\nclass t \Prob(\target).
\end{align}
Noticing that $\rho\ge\tau\alpha^2$ and $\rho_i\ge\tau\beta_{i,j}$ finishes the proof.

\end{proof}

\end{document}